\documentclass{article}

\usepackage[final,nonatbib]{neurips_2022}

\usepackage[utf8]{inputenc} % allow utf-8 input
\usepackage[T1]{fontenc}    % use 8-bit T1 fonts
\usepackage{hyperref}       % hyperlinks
\usepackage{url}            % simple URL typesetting
\usepackage{booktabs}       % professional-quality tables
\usepackage{amsfonts}       % blackboard math symbols
\usepackage{nicefrac}       % compact symbols for 1/2, etc.
\usepackage{microtype}      % microtypography
\usepackage{xcolor}         % colors
\usepackage{subcaption}
\usepackage{graphicx}

\newcommand{\dist}{\mu}

\newtheorem{theorem}{Theorem}[section]

\newtheorem{definition}[theorem]{Definition}

\newtheorem{corollary}[theorem]{Corollary}
\newtheorem{proposition}[theorem]{Proposition}

\newtheorem{observation}[theorem]{Observation}

\newtheorem{claim}[theorem]{Claim}

\newenvironment{proof}{\paragraph{Proof:}}{\hfill$\square$}

\newcommand{\poly}{poly}

\newcommand{\OPT}{OPT} 
 
\newcommand{\Run}{Time}

\newcommand{\R}{\mathbb{R}}

\newcommand{\pr}[1]{\mathop{{\bf Pr}}\left[ #1 \right]}
\newcommand{\prb}[2]{\mathop{{\bf Pr}}_{#1}\left[ #2 \right]}

\usepackage{comment}
\newcommand{\cA}{\mathcal{A}}

\newcommand{\cE}{\mathcal{E}}

\newcommand{\cH}{\mathcal{H}}

\newcommand{\cN}{\mathcal{N}}

\newcommand{\cU}{\mathcal{U}}

\usepackage[framemethod=TikZ]{mdframed}
\newcounter{Frame}
\newenvironment{Frame}[1][h]{%
	\refstepcounter{Frame}
	\begin{mdframed}[%
		frametitle={#1},
		skipabove=\baselineskip plus 2pt minus 1pt,
		skipbelow=\baselineskip plus 2pt minus 1pt,
		linewidth=1.0pt,
		frametitlerule=true,
		]%
	}{%
	\end{mdframed}
}

\newcommand{\eps}{\epsilon}

\title{Stars: Tera-Scale Graph Building for \\ Clustering and Graph Learning}

\author{

CJ Carey\\
Google Research\\
\texttt{cjcarey@google.com}\\
\And
Jonathan Halcrow\\
Google Research\\
\texttt{halcrow@google.com}\\
\And
Rajesh Jayaram\\
Google Research\\
\texttt{rkjayaram@google.com}\\
\AND
Vahab Mirrokni\\
Google Research\\
\texttt{mirrokni@google.com}\\
\And
Warren Schudy\\
Google Research\\
\texttt{wschudy@google.com}\\
\And
Peilin Zhong\\
Google Research\\
\texttt{peilinz@google.com}\\

}

\begin{document}

\maketitle

\begin{abstract}
A fundamental procedure in the analysis of massive datasets is the construction of similarity graphs. Such graphs play a key role for many downstream tasks, including clustering, classification, graph learning, and nearest neighbor search. For these tasks, it is critical to build graphs which are sparse yet still representative of the underlying data. The benefits of sparsity are twofold: firstly, constructing dense graphs is infeasible in practice for large datasets, and secondly, the runtime of downstream tasks is directly influenced by the sparsity of the similarity graph. In this work, we present \emph{Stars}: a highly scalable method for building extremely sparse graphs via two-hop spanners, which are graphs where similar points are connected by a path of length at most two. Stars can construct two-hop spanners with significantly fewer similarity comparisons, which are a major bottleneck for learning based models where comparisons are expensive to evaluate. Theoretically, we demonstrate that Stars builds a graph in nearly-linear time, where approximate nearest neighbors are contained within two-hop neighborhoods. In practice, we have deployed Stars for multiple data sets allowing for graph building at the \emph{Tera-Scale}, i.e., for graphs with tens of trillions of edges. We evaluate the performance of Stars for clustering and graph learning, and demonstrate 10\textasciitilde1000-fold improvements in pairwise similarity comparisons compared to different baselines, and 2\textasciitilde10-fold improvement in running time without quality loss.

\end{abstract}

\section{Introduction}

Given a collection of unlabeled data, a critical technique in many  data mining and machine learning pipelines is the construction of a similarity graph over the data. Similarity graphs provide a sparse representation of the underlying structure of a dataset by creating edges between the most similar data points. This allows downstream analytic procedures to run more efficiently over the data by restricting comparisons to edges in the graph. In particular, many such procedures, such as Hierarchical Agglomerative Clustering (HAC) and variants \cite{dhulipala2021hierarchical,bateni2017affinity,abboud2019subquadratic}, can be run in time nearly-linear in the number of edges.

%For instance, by utilizing similarity graphs, one can reduce the runtime of methods such as Hierarchical Agglomerative Clustering (HAC) \cite{dhulipala2021hierarchical}, and Affinity Clustering \cite{bateni2017affinity} from quadratic to nearly-linear in the number of edges. 

%geometric variants of Hierarchical Agglomerative Clustering (HAC) methods
%, such as single and average linkage clustering, require quadratic time \cite{}, graph-based implementations of these variants are known to run in time nearly linear in the number of edges \cite{dhulipala2021hierarchical}.
%As another example, \rajesh{Add learning based example here}. 
%Affinity Clustering \cite{}, and connected components, 
%Canonical examples include clustering, label propagation, nearest neighbor search, PageRank \cite{}, diversity maximization, and others. 
%\rajesh{Add short paragraph with citations, examples in industry/practice, etc.}

In order to be useful as part of a larger pipeline, graph building itself must be an extremely efficient procedure. Namely, if the time required to construct the graph exceeds the size of the graph itself, then graph building may become the primary bottleneck of the pipeline. To avoid this scenario, graph building algorithms must circumvent all-pairs comparisons by employing selective filtering methods, such as \emph{locality sensitive hashing} (LSH). 
However, with the extraordinary scale of modern datasets, LSH alone is often no longer sufficient to ensure efficiency. Such methods still require an all-pairs comparison within each hash bucket, which is infeasible for the bucket sizes produced by LSH on large datasets. This is not deficit of LSH specifically, but is rather an intrinsic dilemma of scale: even ground truth similarity graphs, such as threshold graphs and $k$-nearest neighbor graphs, would have too many edges to construct efficiently. Instead, graph building at scale requires expanding our conceptions of similarity graphs.

%and then an all-pairs comparison is performed in each bucket.  
%In order to capitalize on the manifold benefits of a similarity graph,
%it is imperative that the graphs be spar
%it is imperative that three conditions be met. Firstly, the resulting graph must be sufficiently sparse, otherwise little is saved in downstream efficiency. Secondly, the representational quality of the graph cannot suffer as a result of the sparsity. Finally, the graph itself must be efficient to construct. The latter condition is paramount, as building a graph by means of all-pairs comparisons between points 

%Many classical algorithms for the analysis of unlabeled data require comparisons to be made between many or all pairs of points to determine the similarity structure of the dataset. However, as the scale of modern data continues to grow at a dramatic rate, it is now all but impossible for to carry out all-pairs comparisons. Instead, for algorithms to keep up with the explosion of data, it is essential that they perform a number of comparisons which is linear or nearly linear in the number of points. This requirement has motivated the proliferation and success of \textit{graph-based} methods, which first construct a similarity graph from the data, and then run graph-based algorithms which perform comparisons only over the edges of the graph.

In this work we go beyond the standard model of a similarity graph, introducing \textit{Stars}, a novel graph building framework for constructing extremely sparse yet high-quality similarity graphs, based on the construction of \textit{two-hop spanners}. Unlike a classical similarity graph, two-hop spanners relax the constraint that similar points should be connected by an edge, instead requiring that such points be connected by a path of length at most two. This allows for considerable improvements in sparsity, scaling up to tens of trillions of edges, with negligible loss in quality for downstream applications. 

Stars employs a specialized locality sensitive hashing technique to bucket points, and then constructs a collection of star graphs within each bucket. 
By creating star graphs, we significantly reduce the number of pairwise comparisons required to process a bucket from quadratic to nearly-linear in the bucket size. In practice, this results in significant (e.g. $10$-$20$ fold) reduction in the number of comparisons made when compared to LSH alone. This reduction yields tremendous speedups for graph construction, especially for similarity measures derived from learned models, such as pre-trained deep neural networks, which are expensive to evaluate. For such learning-based models, similarity computations constitute a majority of the overall runtime, therefore reducing the number of comparisons has enormous performance benefits.

We formally introduce Stars in Section \ref{sec:stars}, and prove theoretical guarantees on its performance for similarity measures such as cosine and Jaccard similarity. For these measures, we prove that Stars can construct two-hop spanners which capture the $1/\epsilon$-approximate $k$-nearest neighbors, for any value of $k$, in time at most $\tilde{O}(n^{1+O(\eps)})$, for any approximation factor $(1/\epsilon) \geq 1$. Additionally, we demonstrate a variant of Stars which, given a threshold $r$, constructs a two-hop spanner such that all points with similarity larger than $r$ are connected by a path of length at most $2$, and no edge connects two points with similarity less than $\eps r$. %are connected.% by an edge in the graph. 
We demonstrate that such two-hop spanners also approximately preserve connected components, allowing for a $2$-approximation of single-linkage clustering (Theorem \ref{thm:connectedComp}). 

We discuss implementation and design details in Section \ref{sec:design}. We then provide an empirical evaluation of Stars in Section \ref{sec:experiments}, where we demonstrate that Stars performs significantly better in practice than even the theoretical bounds suggest. Specifically, in our experiments we set $\epsilon = .99$, and demonstrate that Stars recovers $(1/\epsilon)$-approximate nearest neighbors, as well as yields $(1/\epsilon)$-approximations to threshold similarity graphs, all while constructing significantly fewer than the $O(n^{1.99})$-edges 
(especially for large datasets) suggested by our theoretical guarantees. 
In particular, Stars yields $10$-$20$ fold and $2$-$10$ fold improvements in number of comparisons and runtime respectively, e.g., from 120 trillion to 6 trillion comparisons, when compared to baseline LSH. The improvement is certainly much larger (at least $1000$-fold) compared to the bruteforce (all-pairs) algorithm.
Importantly, these runtime improvements do not negatively impact on the quality of downstream tasks, such as hierarchical clustering. The speedups become even more significant as we employ more sophisticated similarity learning models; moreover, better similarity learning models further improve the quality of the graph.

%\paragraph{SortingLSH} Instead of fixing parameters for the LSH in such a way as to split points uniformly which are larger than some threshold, SortingLSH is an approach to create high-quality neighborhoods in a non-uniform way, so the resulting graph better models the $k$-Nearest Neighbor graph.

%The advantage of SortingLSH is that one does not need to decide on the sketching parameters, namely the scale at which points should be connected, uniformly for all points. Instead, the threshold for which points are connected in the graph is \textit{data-dependent}, and is smaller for points living in denser regions of space, and larger for points which do not have closer near neighbors. Therefore, the SortingLSH is particular well-suited to generate similarity graphs which closely model the $k$-nearest neighbor graph

\subsection{Related Work}

\noindent
\textbf{Graph Building and Nearest Neighbor Search.} The construction of similarity graphs, such as the $k$-nearest neighbor graph, is a key step for many established machine learning and data mining methods  \cite{boiman2008defense,yan2006graph,brito1997connectivity}. To efficiently compute approximate nearest neighbors, Locality Sensitive Hashing (LSH) based methods \cite{gionis1999similarity} and their data-dependent counterparts \cite{andoni2015optimal,indyk2017practical,andoni2018data} have proven tremendously successful. In addition to LSH, other techniques such tree-based data structures \cite{moore2013anchors,liu2004investigation,beygelzimer2006cover}, quantization approaches \cite{guo2016quantization}, and learning based methods \cite{dong2019learning}, have also been utilized. 

 However, the majority of these methods were designed for quickly answering nearest neighbor search queries with a large dataset and a (generally) smaller number of queries, whereas in graph building the query set and the dataset are the same. This scenario allows for an additional set of techniques and optimizations to be employed, such as local search \cite{dong2011efficient}, and the construction of two-hop spanners (as in this work). Perhaps the most similar work to ours is \cite{halcrow2020grale}, which builds large-scale graphs over learned similarity models. In fact, our current results demonstrate that graphs with quality comparable to those in \cite{halcrow2020grale} can be constructed with substantially fewer edges via two-hop spanners.

%\cite{guo2020accelerating} Scann

\noindent
\textbf{Large Scale Graph Clustering} Other than nearest neighbor search, one of the primary downstream analytic tasks run on similarity graphs is clustering. One of the best studied class of such algorithms is hierarchical agglomerative clustering (HAC), which has been analyzed extensively from both a theoretical and practical perspective \cite{dasgupta2016cost,moseley2017approximation,guha2000rock,roy2016hierarchical}. For graph-based HAC, it was recently shown that an approximate clustering can be obtained for average linkage in nearly linear time \cite{dhulipala2021hierarchical}. Previously, this was also shown for Euclidean space using Ward's linkage \cite{abboud2019subquadratic}. Several works also give sub-quadratic HAC algorithms, but sacrifice theoretical guarantees (such as approximation ratio) \cite{cochez2015twister}. 
Finally, a related hierarchical clustering algorithm is the MST-based clustering algorithm Affinity \cite{bateni2017affinity}, a highly-scalable method which we also use for evaluating our two-hop spanners.

%and related methods such as affinity cl
%For points in Euclidean space and Ward's linkage, nearly-linear time bounds were recently shown for approximate HAC \cite{abboud2019subquadratic,dhulipala2021hierarchical }. 
%A related hierarchical clustering algorithm is Affinity clustering, which was introduced in \cite{bateni2017affinity}. 

%\cite{halcrow2020grale}

\noindent
\textbf{Two-Hop Spanners.}  The concept of a two-hop spanner stems from the literature on metric spanners (see the surveys \cite{eppstein2000spanning,zwick2001exact,ahmed2020graph}), which are sparsifications of a graph with the property that distances in specifier well approximate distances in the original graph. Note that a two-hop spanner itself satisfies this property. The first usage of two-hop spanners can be attributed to \cite{har2013euclidean}, who demonstrated the existence of sparse two-hop spanners for Euclidean space with the $\ell_2$ metric. Two-hop spanners have also been implicitly studied in the context of min-size clustering in the MPC model \cite{epasto2022massively}.

\section{Preliminaries}
\label{sec:prelims}
Let $P$ be any set of points, and let $\mu:P^2 \to \R$ be either a \textit{similarity} or a \textit{distance} measure over $P$.  For the remainder, we choose to focus only on similarity measures for clarity of presentation. However, we remark that the techniques we present generalize naturally to distance measures by employing the appropriate locality sensitive hash functions for those measures (see definitions below).

We consider several types of similarity measures in the paper, including the dot-product similarity $\mu(x,y) = \langle x,y \rangle$ between points $x,y \in \R^d$, and the \textit{cosine similarity} $\mu(x,y) = \cos(\theta_{x,y})$, where $\theta_{x,y}$ is the angle between the vectors $x,y$. Additionally, we consider the \textit{Jaccard similarity} between sets: $\mu(A,B) = |A \cup B|/|A \cap B|$, where $A,B \subset \cU$ for some universe $\cU$, and  the \textit{weighted Jaccard similarity} between non-negative vectors $x,y$, given by by $\mu(x,y) = \frac{\sum_i \min(x_i, y_i)}{\sum_i \max(x_i, y_i)}$.
Finally, we use \textit{learned similarity measures} $\mu$ which are computed by a pre-trained model, such as a deep neural network, and may be expensive to evaluate.

%In contrast with measures such as the dot-product similarity, for models such as deep neural networks each evaluation of $\mu$ can be rather expensive. Therefore, in this case it is even more critical to minimize the number of similarity comparisons. 
\noindent
\textbf{$\mu$-Nearest Neighbors.}
%Given a point $p \in P$ and real $r \in \R$, and a similarity measure $\mu:P^2 \to \R$ over $P$, we define the similarity ball $\cB(p,r) = \{x \in P \; | \; \mu(x,p) > r\}$. 
Given $(P,\mu)$ where $n  =|P|$, for any $p \in P$ we define the nearest neighbor ordering $\pi_p:[n] \to P$ as any bijection which satisfies $\mu(p,\pi_p(1)) \geq \mu(p,\pi_p(2)) \geq \dots \geq \mu(p,\pi_p(n))$. We write $\tau_i(p) = \mu(p,\pi_p(i))$ to denote  the similarity to the $i$-th nearest neighbor, and we denote the set of $i$-nearest neighbors via $N_i(p) = \{\pi_p(1),\dots,\pi_p(i)\}$.

%if $\mu$ is a dissimilarity measure we define the ball $\cB(p,r) = \{x \in P \; | \; \mu(x,p) \leq r\}$. Conversely, for similarity measures $\mu$, we define the similarity ball $\cB(p,r) = \{x \in P \; | \; \mu(x,p) > r\}$. For any point set and (dis)similarity measure tuple $(P,\mu)$, with $n  =|P|$, and for any $p \in P$, define the nearest neighbor ordering $\pi_p:[n] \to P$ as any bijection which satisfies $\mu(p,\pi_p(1)) \leq \mu(p,\pi_p(2)) \leq \dots \leq \mu(p,\pi_p(n))$ for dissimilarity measures, or $\mu(p,\pi_p(1)) \geq \mu(p,\pi_p(2)) \geq \dots \geq \mu(p,\pi_p(n))$ for similarity measures. In both cases, we define the distance (or similarity) to the $i$-th nearest neighbor as $\tau_i(p) = \mu(p,\pi_p(i))$, and we denote the set of $i$-nearest neighbors via $N_i(p) = \{\pi(1),\dots,\pi(i)\}$. 

\noindent
\textbf{Locality Sensitive Hash Families.} We now introduce the notion of a locality sensitive hash (LSH) family $\cH$ for similarity measures. Given a similarity measure $\mu$ over a point set $P$, an LSH family $\cH$ is a family of hash functions $h: P \to U$, for some universe $U$, with the property that $\prb{h \sim \cH}{h(p) = h(q)}$ is larger when the points $p,q$ are more similar, namely when $\mu(p,q)$ is large. For ease of presentation, we use a simplified parameterization of LSH families, as opposed to the more standard parameterization of $(r_1,r_2,p_1,p_2)$-sensitive families (see, e.g. \cite{har2012approximate}).% which captures the essential properties which we need for efficient graph building. 

\begin{definition}[$(r_1,r_2,\rho)$-sensitive Family]\label{def:LSH}
Fix any parameters $r_1 < r_2 \in \R$, and $\rho \in [0,1]$. Let $P$ be any point set with similarity measure $\mu: P^2 \to \R$, and fix a universe of buckets $U$.  Then a family $\cH$ of hash functions $h: P \to U$ is said to be a $(r_1,r_2,\rho)$-sensitive locality sensitive hash family for $(P,\mu)$ if the following holds: for any $p,q \in P$, if $\mu(p,q) > r_2$, then $\prb{h \sim \cH}{h(p) = h(q)} \geq n^{-\rho}$, and if $\mu(p,q) < r_1$ then $\prb{h \sim \cH}{h(p) = h(q)}  < n^{-4}$.
%\begin{itemize}
 %   \item For any $p,q$ with $\mu(p,q) > r_2$, we have: $\prb{h \sim \cH}{h(p) = h(q)} \geq n^{-\rho}$.
  %  \item For any $p,q$ with $\mu(p,q) < r_1$, we have: $\prb{h \sim \cH}{h(p) = h(q)}  < \frac{1}{n^4}$.
%\end{itemize}
\end{definition}
%We remark that, despite the differences in parameterization, given a $(r_1,r_2,p_1,p_2)$-sensitive hash family, we can obtain a $(r_1,r_2,\rho)$-sensitive family, with $\rho = \frac{\log p_1^{-1}}{\log p_2^{-1}} \cdot \log_n \frac{1}{p_1}$, by simply concatenating sufficiently many copies of the hash function. 

\noindent
\textbf{Graph Building and Similarity Graphs.} %Given a point set $P$ and a similarity measure $\mu$ over $P$, the goal of graph building is to construct a graph $G = (P,E)$ over $P$, where similar points are connected with edges.
In this work, we consider two distinct notions of similarity graphs, which are appropriate in different contexts based on the application. The first notion is that of an $r$-threshold graph, which uniformly connects points with similarity above a given threshold.

\begin{definition}[$r$-threshold graph]\label{def:thresholdGraph}
	Given a point set $P$, similarity measure $\mu$, and a threshold $r$, the $r$-threshold graph for $(P,\mu)$ is the graph $G=(P,E)$ where $E = \{ (x,y) \in P \; | \; \mu(x,y) \geq r\}$. 
\end{definition}

The $r$-threshold graph for $(P,\mu)$ is a useful similarity graph when downstream tasks impose a uniform threshold on when points should be connected. For instance, clustering with the constraint that pairs of points within a cluster have similarity above a threshold $r$. %, then running average or single-linkage clustering on the $r$-threshold graph is an ideal choice. 
On the other hand, from the perspective of nearest neighbor search, one would like edges to be created \textit{non-uniformly}, where a point $p$ is connected to its $k$ closest neighbors. Formally:

%where the threshold used for the construction of edges to a point $p$ depends on how close other points are to $p$. %Specifically, one often wants to create edges between each point $p$ and its closest neighbors. A canonical example includes downstream tasks based on nearest neighbor search. In these settings, the $k$-nearest neighbor graph is appropriate. Formally:

\begin{definition}[$k$-near neighbor graph ($k$-NN graph)]\label{def:knnGraph}
	Given a set of $n$ points $P$ and similarity measure $\mu$, let $\pi_p:P \to [n]$ be a nearest neighbor ordering over $P$. Then the $k$-nearest neighbor graph for $(P,\mu)$ is the graph $G=(P,E)$ where $E = \cup_{p \in P} \{(p,q)\mid q\in N_k(p)\}$. 
\end{definition}

\noindent
\textbf{Two-Hop Spanners.}  A two-hop spanner is a similarity graph which relaxes the guarantees of the prior two notions of similarity graphs, where similar points are only required to be connected by a path of length $2$, rather than by a direct edge. The advantage of such spanners is that they can be considerably sparser than the associated similarity graph, and still produce high-quality results for downstream tasks (such as clustering). 
 For any graph $G = (V,E)$, $v \in V$, and integer $k \geq 1$, we write $\cN_k(p)$ to denote the $k$-hop neighborhood of $p$: namely, the set of vertices $u \in V$ which are at (unweighted) shortest-path distance at most $k$ from $v$ in $G$. The following definition captures the generalization of $r$-threshold graphs from one-hop to two-hop neighborhoods. 

\begin{definition}[Threshold Two-Hop Spanners]\label{def:twohop}
	Given a point set $P$ and similarity measure $\mu$, a $(r_1,r_2)$-two-hop spanner is a graph $G=(P,E)$, such that: $(1)$ for every edge $(p,q) \in E$, we have $\mu(p,q) \geq r_1$, and $(2)$ for every pair $p,q\in P$ with $\mu(p,q) \geq r_2$, we have $p \in \cN_2(q)$. 
%	\begin{enumerate}
%		\item For every edge $(p,q) \in E$, we have $\mu(p,q) \geq r_1$. 
%		\item For every pair $p,q\in P$ with $\mu(p,q) \geq r_2$, we have $p \in \cN_2(q)$. 
%	\end{enumerate}
\end{definition}
The above notion of two hop-spanners provide a good (two-hop) approximation of the $r$-threshold graph for a point set. The approximation is controlled by the gap between $r_1,r_2$. %Specifically, given a threshold $r$, ideally one would like to construct a $(r',r)$-2-hop spanner where this gap is as small as possible. 
In Section \ref{sec:lsh}, we give provable guarantees for Stars, by demonstrating that it constructs a two-hop spanners where this gap is small.
In Section \ref{sec:sortingLSH}, we define an analogous notion of two-hop spanners for approximation the $k$-nearest neighbor graph, and prove that Stars efficiently constructs such an approximation for $k$-nearest neighbors as well. 

Two-hop spanners can be used as a powerful approximation of the underlying similarity structure of dataset. For instance, we demonstrate that two-hop spanners can be utilized to obtain an approximation to the optimal $k$-single linkage clustering solution, which is to partition $P$ into $k$ clusters $C_1,\dots,C_k$ so as to minimize  $\max_{i \neq j} \max_{p\in C_i,q\in C_j} \mu(p,q)$.% (i.e., minimize the maximum intra-cluster similarity). 
\begin{theorem}\label{thm:connectedComp}
	Let $c\geq 1$ and $r<\OPT_{k}/c$, where $\OPT_{k}$ is optimal cost of $k$-single-linkage clustering on $(P,\mu)$.  Then any $(r/c,r)$-2-hop spanner $G$ has at least $k$ connected components.
	Furthermore, for any two connected components $C,C'$ of $G$, $\min_{x\in C,y\in C'}\mu(x,y)\geq r$. Thus, by constructing $\log\left(\frac{\max_{x,y \in p} \mu(x,y)}{\min_{x,y \in P} \mu(x,y)}\right)$ distinct $(r/c,r)$-2-hop spanner for geometrically increasing $r$, one obtains a $2$-approximation to $k$-single-linkage clustering. 
\end{theorem}
%in section \ref{sec:connectedComponents},
\vspace{-.1 in}
\vspace{-.05 in}
\section{Graph Building via Two-Hop Spanners}\label{sec:stars}
In this section, we theoretically analyze the performance of the Stars algorithm, by demonstrating that it efficiently constructs highly-sparse two-hop spanners. We describe two variants of Stars, in 
Sections \ref{sec:lsh} and \ref{sec:sortingLSH} respectively. The first can be used to approximate the $r$-threshold similarity graph, and the second can be used to approximate the $k$-nearest neighbor graph.

\subsection{Two-Hop Spanners via Locality Sensitive Hashing}\label{sec:lsh}
We first describe the Stars algorithm for constructing a two-hop spanner for approximating $r$-similarity threshold graph for a point set $P$ and similarity measure $\mu$. Specifically, given a locality sensitive hash family $\cH$ for $(P,\mu)$, Stars proceeds by bucketing points via a randomly drawn hash function $h \sim \cH$. Within each bucket, Stars samples a random ``leader'' $p$, and then connects together all points in the bucket to $p$ which are sufficiently similar to $p$, effectively creating a star graph centered at $p$. The process is repeated with $R$ independent sketches, to ensure that all sufficiently similar points land in the same bucket at least once. The formal algorithm is given in the algorithm \hyperlink{twohopLSH}{Stars 1}.

\begin{figure}
\begin{Frame}[\hypertarget{twohopLSH}{Stars 1}: Constructing Approximate Threshold Graphs]
    \textbf{Input:} Point set $P$, similarity measure $\mu$, and a  $(r_1,r_2,\rho)$-sensitive hash family $\cH$.\\
    
   \noindent \textbf{Repeat:} the following procedure $R = c_1 n^\rho \log n$ times, for a sufficiently large constant $c_1$.  

    \begin{enumerate}
      \item Evaluate $h(p)$ for each $ p \in P$, and construct the buckets $B_u  = \{p \in P \; | \; h(p) = u \}$.
      \item For each bucket $B_u$, where $u \in U$:
      \begin{itemize}
    %  	\item Construct the hash buckets $B_u = h^{-1}(u) = \{p \in P \; | \; h(p) = u \}$.
          \item Sample a uniformly random leader $x \sim B_u$. 
          \item For all $y \in B_u \setminus \{x\}$,  if $\mu(x,y) > r_1$, then create an edge $(x,y)$ with weight $\mu(x,y)$. 
      \end{itemize}
 
    \end{enumerate}
\end{Frame}
%\caption{The Two-Hop Spanner + LSH graph building algorithm}\label{alg:twohopLSH}
\end{figure}

%\rajesh{Change this theorem so that edges between two-hop neighbors are large similarity (i.e., a star on the full graph should not satisfy this)}
\begin{theorem}\label{thm:lshMain}
Let $P$ be a point set equipped with a similarity measure $\mu$, and fix any $r_1 < r_2 \in \R$. Let $\cH: P \to U$ be a $(r_1,r_2,\rho)$ family of hash functions, such that each $h \in \cH$ can be evaluated in time at most $\Run(\cH)$.
Let $G = (P,E)$ be the graph produced by the algorithm \hyperlink{twohopLSH}{Stars 1}, using the LSH family $\cH$. Then with probability $1-1/\poly(n)$, $G$ is a $(r_1,r_2)$-two-hop spanner for $(P,\mu)$. Moreover, the number of edges produced is at most $|E(G)| =\tilde{O}(n^{1+ O(\rho)})$, and the runtime of the algorithm is bounded by $\tilde{O}(n^{1+ O(\rho)} \cdot \Run(\cH))$. 
\end{theorem}
In the Appendix, we demonstrate that for parameters $\alpha ,\eps > 0$, there exist $(1-\eps^{-1}\alpha, 1-\alpha, O(\eps))$-sensitive hash families for both cosine and Jaccard similarity measures.

\subsection{Two-Hop Spanners via SortingLSH}\label{sec:sortingLSH}
%\rajesh{need to motivate sortingLSH}

We now describe a variant of the Stars algorithm for the construction of two-hop spanners which approximate that $k$-nearest neighbor graph ($k$-NN graph, Definition \ref{def:knnGraph}). While the $k$-NN Graph itself is sparse for small $k$, for large $k$ this is no longer the case, and directly approximating the $k$-NN graph becomes infeasible. %, due to the $\Omega(nk)$ edges which would be required. 
Instead, by employing two-hop spanners, we show how one can construct extremely sparse graphs, with a almost-linear number of edges, such that the approximate $k$-nearest neighbors of every point $p$ are contained in the two-hop neighborhood of $p$. To do so, we will need to utilize a technique known as SortingLSH.

\paragraph{SortingLSH.} Since the similarities $\tau_k(p)$ between points $p$ and their $k$ nearest neighbors may vary significantly across the dataset, we cannot apply a single level of bucketing based on a hash family chosen with a fixed set of parameters, as doing so would result in points being split around a uniform threshold. Instead, we utilize a technique known as \textit{SortingLSH}, which originates from the work \cite{bawa2005lsh}. %, and has recently been used in \cite{Epasto2021clustering} for privacy applications.
SortingLSH evaluates a sequence of hash functions $H(p) = (h_1(p),\dots,h_\ell(p))$ for each point $p \in P$, and interprets the resulting string of buckets as an key for $p$. One then sorts the keys $H(p)$ lexicographically, and breaks the sorted sequence into contiguous chunks of a given window size $W$. We then apply the two-hop spanner technique on each chunk, by sampling $s$ leaders within that bucket, and comparing each leader to the rest of the chunk. 

The key advantage of SortingLSH is that points $p$ living in dense regions of the similarity space $(P,\mu)$ (i.e., the similarity $\tau_k(p)$ is large) will share a longer prefix of hashes with their $k$-nearest neighbors, and therefore be more likely to be placed in the same window as them. At the same time, points $p$ whose $k$-nearest neighbors are not as similar will still be equally as likely to share a (shorter) prefix with their neighbors, and therefore also placed in the same window. 

In order to describe our results, we first introduce a notion of $k$-approximate nearest neighbors ($k$-ANN), defined with respect to a family of hash functions $\cH$. In what follows, recall that given any $p \in P$ and $i \in [n]$ we write $\pi_p(i)$ to denote the $i$-th nearest neighbor to $p$ with respects to $\mu$.

\begin{definition}[ANN with respects to an LSH Family]\label{def:generalANN}
	Fix a point set $P$ equipped with a similarity measure $\mu$, and let $\cH = \{h : P \to U\}$ be a family of hash functions. Fix any $\rho \in [0,1]$.  We say that a collection of sets $\cA = \{\cA_p\}_{p \in P}$ is an $(k,\rho)$-ANN family with respect to $\cH$ if for every $p \in P$:
\begin{itemize}
    \item $\cA_p$ is a prefix of the sequence $\pi_p(1),\pi_p(2),\dots,\pi_p(n)$ with size $|\cA_p| \geq k$. 
    \item There exists an integer $\ell = \ell(p)$ such that for all $i \in [k]$, we have: 
    \[\pr{(h_1(p),\dots,h_\ell(p)) = (h_1(\pi(i)),\dots,h_\ell(\pi(i)) } \geq n^{-\rho} \] 
    and such that for all $x \notin \cA_p$ we have 
 \[\pr{(h_1(p),\dots,h_\ell(p)) = (h_1(x),\dots,h_\ell(x)) } \leq n^{-4} \] 
\end{itemize}

%When the hash family $\cH$ is understood from context, we simply say that $\cA_p$ is a $(k,\rho)$-ANN set for $P$.
%We say that a collection $\cA = \{\cA_p\}_{p \in P}$ is a $(k,\rho)$-ANN set if for each $p \in P$, $\cA_p$ is a $(k,\rho,M)$-ANN set for $p$. 
If $\cA$ has the property that $\ell(p) \leq M$ for all $p \in P$, then we say that $\cA$ is $M$-bounded $(\rho,M)$-ANN family (hereafter a $(k,\rho,M)$-ANN family).
\end{definition}

Definition \ref{def:generalANN}, while at first somewhat unwieldy, is fairly straightforward to unpack. Specifically, given a base hash family $\cH$, and let $\cH^\ell$ denote the family of $\ell$-wise concatenations of $h \in \cH$. Then if, for every $p \in P$, there exists a smaller threshold $r_p \leq \tau_k(p)$ and a sketch length $\ell(p)$ such that the family $\cH^{\ell(p)}$ is a $(r_p,\tau_k(p),\rho)$-sensitive hash family (Definition \ref{def:LSH}), then the family $\cA = \{\cA_p\}_{p \in P}$, where $\cA_p = \{q \in P \; | \; \mu(p,q) \geq r_p\}$, is a  $(k,\rho)$-ANN family. Note that the gap between $r_p$ and $\tau_k(p)$ will depend on the properties of the family $\cH$. Lastly, the condition of being $M$-bounded simply ensures that one can reach the appropriate splitting point $\ell(p)$ with a limited number of hash functions. Thus, for points $p$ with extremely large $\tau_k(p)$ (e.g.., near-duplicate neighbors), in order to ensure that $\cA$ is $M$-bounded, one may need to limit $\ell(p) = M$ which would result in a value of $r_p$ that is smaller than what would otherwise be possible with a larger sketch length $\ell(p)$.

We now prove explicit bounds for the above notion of $k$-ANNs for the case of angular (cosine) similarity, and the Jaccard similarity measure. The bounds utilize the well-known SimHash family \cite{C02} for cosine similarity, and MinHash \cite{B97} for Jaccard similarity. For the case of weighted Jaccard similarity, we remark that there is a straightforward reduction (i.e. similarity preserving mapping) from the (integer) weighted to unweighted case by simply duplicating coordinates, therefore the following bounds also apply to the weighted variant. For weighted Jaccard distance over general non-integer vectors, one can instead use the variant of min-hash for probability distributions of~\cite{moulton2018maximally}. 

\begin{proposition}[Approximate Nearest Neighbors for Angular and Jaccard Similarity]\label{prop:simhashandJaccard}
Let $P$ be a subset of either $\R^d$ or $2^{\cU}$ for some universe $\cU$. In the first case, let $\mu$ be the Angular Similarity $\mu(x,y) = 1- \theta_{x,y}$ where $\theta_{x,y}$ is the angle between $x,y \in P$ (normalized so that $\theta_{x,y} \in [0,1]$), and in the second case we set $\mu$ to be the Jaccard similarity $\mu(A,B) = \frac{|A \cap B|}{|A \cup B|}$ between sets $A,B \subset \cU$. For any $p \in P$, let $s_k(p) = 1- \tau_k(p)$. 
Then for any $\eps \in (0,1)$ and $M \geq 1$,  there exists a hash family $\cH$ for $(P,\mu)$ such the family $\cA = \{\cA_p\}_{p \in P}$, where
\[  \cA_p= \left\{ x \in P \; | \; \mu(p,x) \geq \min\left\{ 1- \eps^{-1} s_k(p), \; \; 1-1/M \right\} \right\} \]
is a $(k, O(\eps), 4 M  \log n)$-ANN family with respect to $\cH$.
\end{proposition}

Since both the angular and Jaccard similarity are bounded by $1$, the value $s_k(p)$ in the Proposition \ref{prop:simhashandJaccard} can be interpreted as a natural \textit{dissimilarity} measure $d(p,q) = 1- \mu(p,q)$. 
Therefore, Proposition \ref{prop:simhashandJaccard} demonstrates that we can take $\cA_p$ to be the set of points $q$ which are at most a $1/\eps$ factor farther away than the $k$-th nearest neighbor to $p$, namely $\mu(p,q) \leq \eps^{-1} d(p,\pi_k(p))$. 

We now state the main result of this section, which demonstrates that Stars provably recovers nearly $k$ distinct $(1/\eps)$-approximate $k$-nearest neighbors in the two-hop neighborhood of $p$. 
%We now state the main result of this section, which demonstrates that Stars with SortingLSH provides a good approximation of the $k$-NN similar graph.
Formally, for  each point $p \in P$, we guarantee that there is a two-hop path from $p$ to nearly $k$ points $q \in \cA_p$. Moreover, we guarantee that this path utilizes only edges between other approximate nearest neighbors $u \in \cA_p$.

% % by demonstrating that for $\eps = .99$, Stars can recover the $\frac{100}{99}$-approximate $k$-nearest neighbors using significantly fewer than $O(n^{1.99})$ comparisons.As previously noted, we  demonstrate in Section \ref{sec:experiments} that Stars performs significantly better in practice than even the theoretical bounds suggest.

\begin{figure}
\begin{Frame}[\hypertarget{twohopSortingLSH}{Stars 2}: Constructing Approximate Nearest Neighbor Graphs]
    \textbf{Input:} Point set $P$, a hash family $\cH$, ANN parameters $k,\rho,M > 0$, recall parameter $\delta \in (0,1)$.\\
    
   \noindent \textbf{Repeat:} the following procedure $R = c_1 n^\rho \log n$ times, for a sufficiently large constant $c_1$.  

    \begin{enumerate}
      \item Let $(h_1,h_2,\cdots, h_M)$ be a sequence of independent draws from $\mathcal{H}$. 
        \item Sort the points $x$ lexicographically according to the hash values $(h_1(x),\dots, h_M(x))$. Let $x_1,\dots,x_n$ be this ordering. Set the window size to be $W=16k$
        \item Pick a random shift $r \sim [W/2,\dots,W]$, and split $x_1,\dots,x_n$ into consecutive blocks $B_1,B_2,\dots,B_{\ell}$, each of size at most $W$, such that $B_1 = \{x_1,\dots,x_r\}$. 
        \item \textbf{If:} $k > n^{2\rho}$
        \begin{enumerate}
        \item For each block $B_j$, sample $s = c_2 \delta^{-1} n^{ \rho} \log^2 n$ uniformly random leaders $c_2$ $y_1^j,\dots,y_s^j \sim B_j$, where $c_2$ is a sufficiently large constant depending only on $c_1$.
        \item Create an edge $(x,y_i^j)$ with weight $\mu(x,y_i^j)$ for every $x \in B_j$ and leader $y_i^j \in B_j$
        \end{enumerate}
          \item \textbf{Else if:} $k \leq n^{2\rho}$
        \begin{enumerate}
        \item For each block $B_j$, create an edge $(x,y)$ with weight $\mu(x,y)$ for every pair $x,y \in B_j$.
        \end{enumerate}
    \end{enumerate}
\end{Frame}
%\caption{The Two-Hop Spanner + SortingLSH graph building algorithm}\label{alg:twohopSortingLSH}
\end{figure}

%The main challenge towards the proof of Theorem \ref{thm:sortingLSHMain} is that, while an LSH family $\cH$ guarantees that the near 

\begin{theorem}\label{thm:sortingLSHMain}
Let $P$ be a point set equipped with a similarity measure $\mu$, and fix any $\delta \in (0,1)$. Let $\cH: P \to U$ be a family of hash functions, such that each $h \in \cH$ can be evaluated in time at most $\Run(\cH)$.
Let $G = (P,E)$ be the graph produced by the algorithm \hyperlink{twohopSortingLSH}{Stars 2}. Let $\cA = \{\cA_p\}_{p \in P}$ be a $(k,\rho,M)$-ANN family with respects to $\cH$. For any $p \in P$, let $\hat{G}_{p}$ denote the subgraph of $G$ induced by the set $\cA_p \cup \{p\}$.  Then with probability $1-1/\poly(n)$, for every $p \in P$, we have: $| \mathcal{N}^2_{\hat{G}_p}(p) |  \geq (1-\delta)k$. Moreover, we have $|E(G)| =\tilde{O}(\delta^{-1} n^{1+ O(\rho)})$, and the runtime of the algorithm is $\tilde{O}(\delta^{-1}n^{1+ O(\rho)}  M \cdot \Run(\cH))$. 
\end{theorem}

Theorem \ref{thm:sortingLSHMain} demonstrates that, even for extremely large values of $k$, for instance $k = \sqrt{n}$, one can construct a two-hop spanner with a nearly-linear number of edges, such that each point $p$ is connected with at least $(1-\delta)k$ of its $k$-approximate nearest neighbors, via a path containing only $k$-approximate nearest neighbors of $p$. Notably, the runtime and the size of the graph produced are \textit{independent} of $k$.

\section{System Implementation and Design}\label{sec:design}
We have implemented Stars as part of
%\textit{[redacted for blind review]}
 the Grale \cite{halcrow2020grale} graph building system using Flume - a C++ counterpart to FlumeJava \cite{flume}.
which is based on the Adaptive Massively Parallel Computation (AMPC) model \cite{ampc}.
Each logical unit of computation is automatically distributed across a number of worker machines,
with the experiments in this paper scaling to thousands of individual workers.

Our implementation has two primary phases: generating LSH tables using LSH or SortingLSH,
then scoring pairs of points that share a sketch using all-pairs or Stars.
For efficiency we generate LSH tables containing only the identifier of each point, excluding the associated point features,
as the same point is sketched $R > 1$ times to aid in edge recall.
To compute pairwise similarity we thus require an additional round of communication join point features with the LSH tables.
We implement this in one of two ways: a MapReduce-style distributed shuffle sort, or via lookups in a distributed hash table (DHT).
These two options trade off CPU time and disk usage for memory.

In the shuffle case, we require $O(R n)$ additional disk storage and $O(R n \log(R n))$ time to materialize the joined table.
For problem settings with billions of points and large feature vectors, this extra storage requirement can be prohibitive.
The DHT caches the entire input dataset in memory across multiple machines, requiring $O(n)$ RAM but no additional on-disk storage.
This enables online feature lookup as we process each bucket, eliminating the need for a shuffle and costly disk I/O.

An additional important implementation detail is the limit we impose on LSH bucket sizes.
In the worst case, a naive LSH implementation could have the same or worse running time than a brute-force all pairs comparison,
because a poorly chosen LSH function could hash the entire dataset to a single value, up to $R$ times.
To ensure robustness to sub-optimal LSH settings,
we randomly partition large buckets into size-constrained sub-buckets prior to pairwise scoring.
This reduces the overall number of edges eligible for comparison, potentially impacting edge recall,
but also capping the worst case running time for the scoring phase of graph construction.
Due to its nearly-linear runtime complexity, the Stars algorithm enables us to relax the sub-bucket size limitation significantly.

See Appendix~\ref{sec:additional_impl_details} for further implementation details.

\section{Empirical Study}\label{sec:experiments}

We evaluate the performance and quality of Stars on both real and synthetic datasets of varying scales. We compare the LSH+Stars, SortingLSH+Stars, SortingLSH, and bruteforce (AllPair) algorithms.

\noindent
\textbf{Datasets.} We run experiments on three real public datasets: MNIST~\cite{lecun1998gradient}, Amazon2m~\cite{Bhatia16} (also known as OGBN-Products~\cite{OGB}), and Wikipedia~\cite{wikipedia}, and two synthetic datasets: random1B and random10B.
MNIST contains 60k data points, each of which has a feature of a 784-dimensional float vector.
Wikipedia contains 3,650,339 data points, where each point is represented by a set of strings with positive weights.
Amazon2m contains 2,449,029 data points, each of which has a feature of a 100-dimensional float vector and a set of strings.
Random1B and random10B are generated from a Gaussian mixture model with 100 modes, where each data point has 100 dimensions.
Random1B contains $10^9$ data points and Random10B contains $10^{10}$ data points.
We refer readers to Appendix~\ref{sec:additional_exp} for further details regarding the datasets.
%We demonstrate the effectiveness of Stars in two ways: the runtime improvements that come from it over standard LSH while also demonstrating that we still achieve similar graph quality for clustering tasks. 

\noindent
\textbf{Sketching parameters.} 
%We set number of sketches $R=25,100,400$ for all LSH and SortingLSH based algorithms.
%For Stars, we set number of leaders $s=25$. 
%For SortingLSH-based algorithms, we set window size $W=250$ and the sketching dimension $M=30$.
For each similarity measure studied, we use a corresponding LSH function to build the graph.
In particular, for MNIST, Random1B and Random10B datasets, we study the cosine similarity between float vector features, and thus we employ SimHash for them. 
%For LSH-based algorithms, we use sketching dimension $M=12$ for MNIST and $M=16$ for Random1B and Random10B.
For the Wikipedia dataset, we study the weighted Jaccard similarity between two sets of strings with weights, and therefore we use the weighted Minhash LSH to build the graph. %and the LSH based algorithms use sketching dimension $M=3$.
For the Amazon2m dataset, we study two different similarity measures: %one is based on a mixture of cosine similarity and Jaccard similarities, and the other one is a neural network trained on top of 
(1) %The first similarity measure is computed by $0.8\times$ the cosine similarity between the float vector features $+$ $0.2\times$ the Jaccard similarity between two sets of strings. 
a mixture of cosine similarity and Jaccard similarity, and
%Thus, we use a corresponding mixture of SimHash and MinHash to build the graph.
(2) a neural network where the training set of candidate pairs are generated by SimHash over float vector features and MinHash over sets of strings~\cite{halcrow2020grale}. 
In summary, in both cases, we use a mixture of SimHash and MinHash for graph building.
%The LSH based algorithms use sketching dimension $M=12$ for the mixture of SimHash and MinHash.
See Appendix~\ref{sec:additional_exp} for more detailed sketching setups.

\noindent
\textbf{Number of comparisons.} In Figure~\ref{fig:num_comparisons}, we illustrate the number of pairwise similarity comparisons of each algorithm on each dataset.
In all cases, Stars yields at least a $\sim10$-fold improvement in number of comparisons over the other algorithms. 
%The improvement is even larger ($1000$-fold) if comparing with the bruteforce algorithm. 
The number of comparisons used by Stars can be further reduced by choosing a smaller number of leaders (see Appendix~\ref{sec:additional_exp}).

\begin{figure*}
    \centering
         \includegraphics[width=\textwidth]{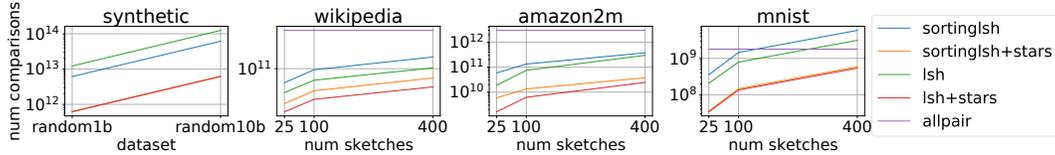}
     \vspace{-2mm}
        \caption{\small Number of comparisons of each algorithm on each dataset. For Random1B and Random10B, we only run algorithms with number of sketches $R=25$, and the AllPair algorithm does not finish in $3$ days.  }
        \label{fig:num_comparisons}
\end{figure*}

\begin{figure*}
   \centering
\begin{tabular}{ccc}
\small
\includegraphics[width=\textwidth]{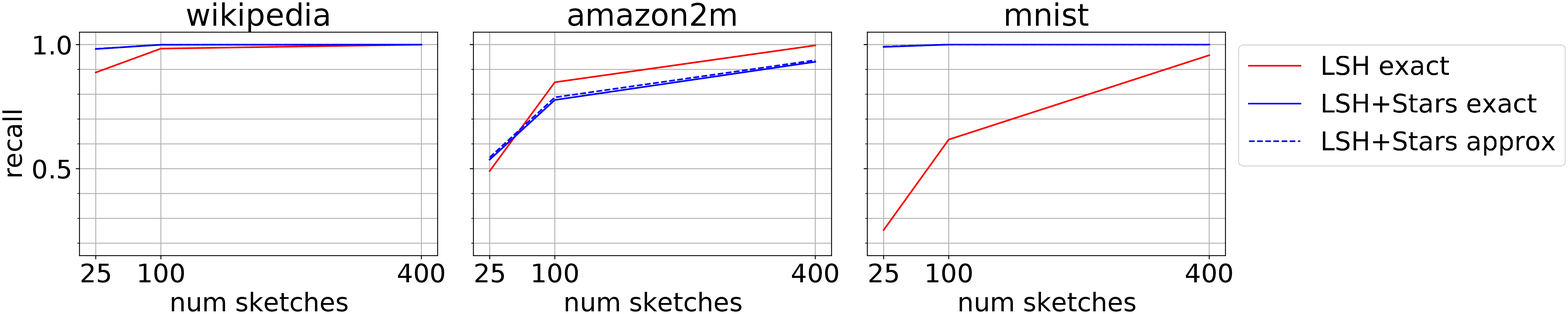}\\
\hspace{5mm}\includegraphics[width=\textwidth]{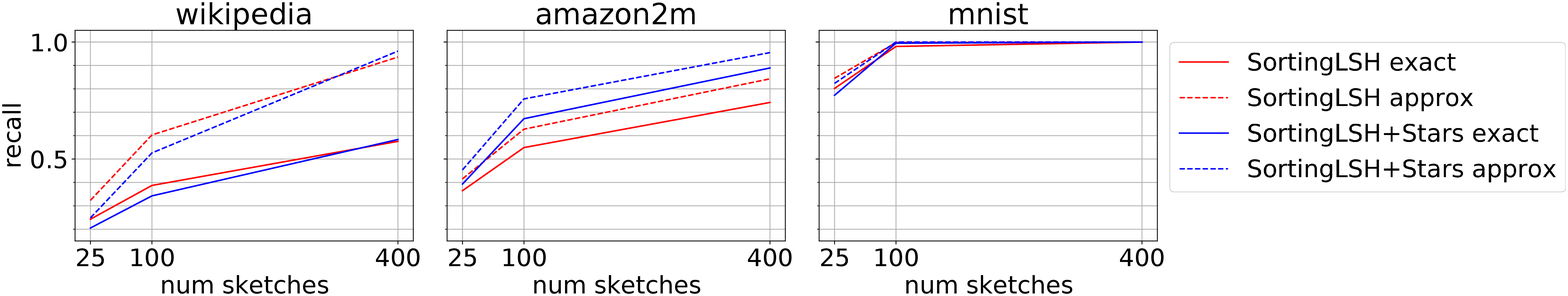}
\end{tabular}
    \caption{\small The recall of found near(est) neighbors of each algorithm.}
    \label{fig:degree_recall} 
    \vspace{-5mm}
\end{figure*}

\begin{figure*}
   \centering
\begin{tabular}{ccc}
\small
\includegraphics[width=0.28\textwidth]{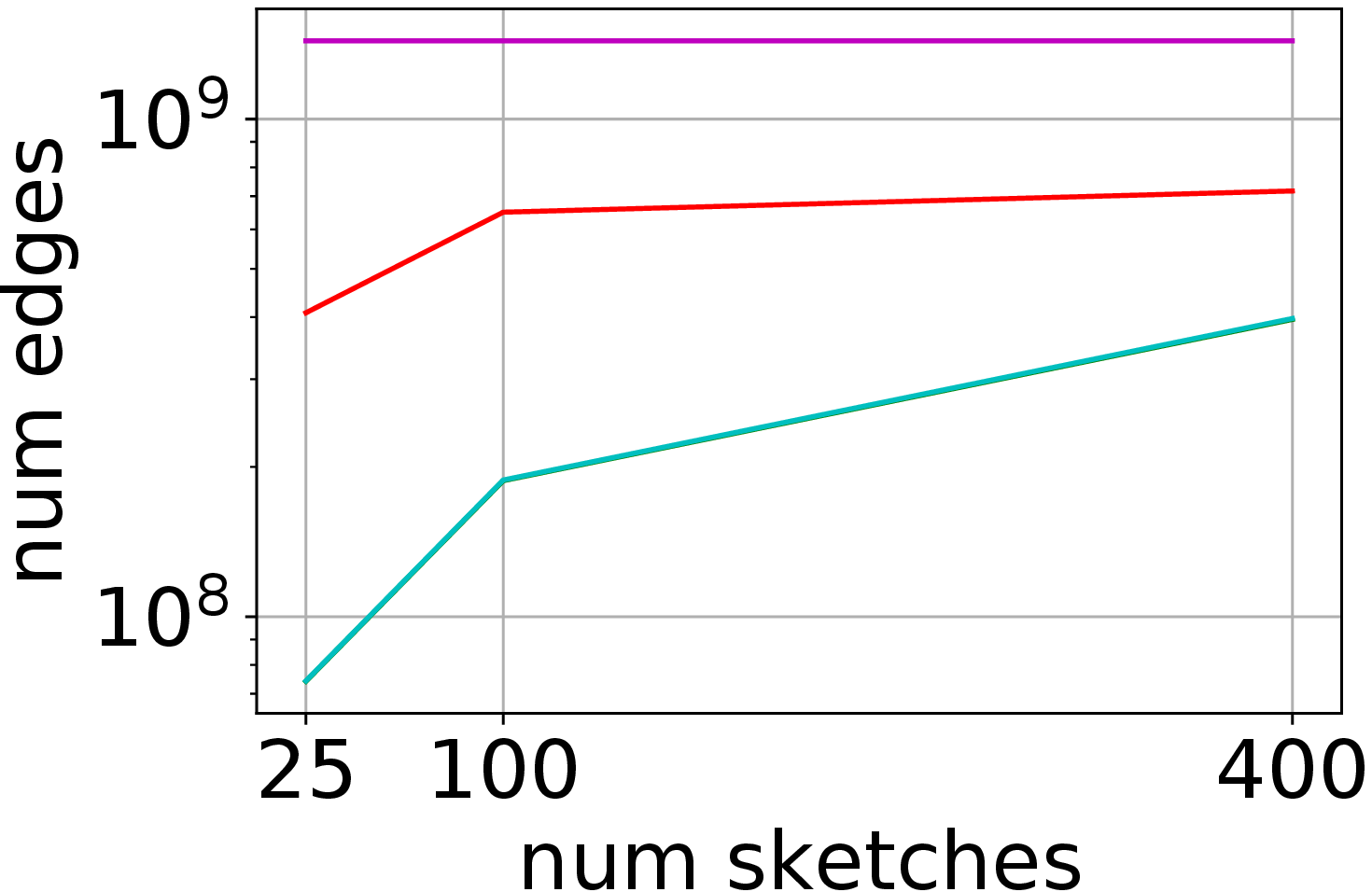}&
\includegraphics[width=0.26\textwidth]{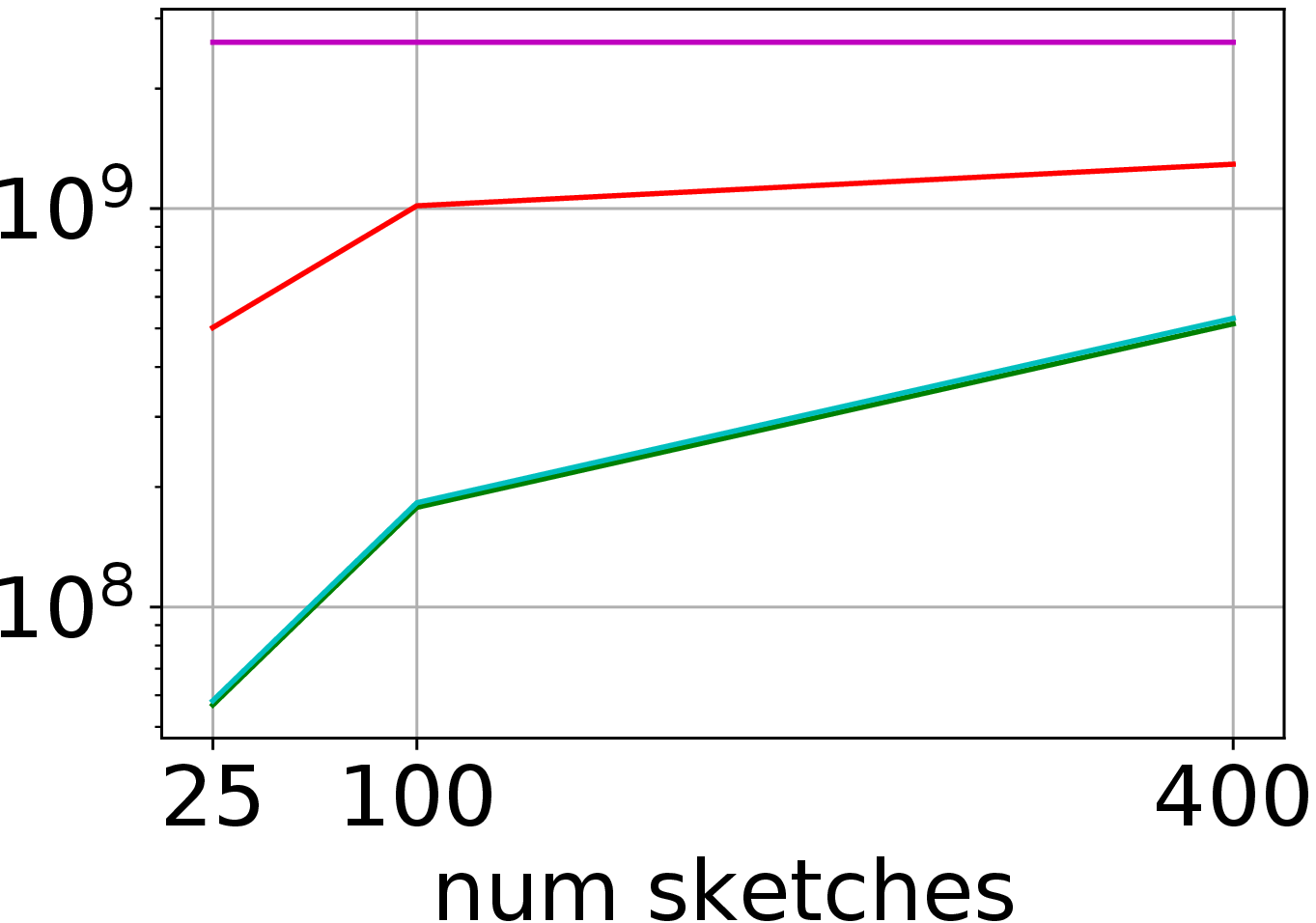}&
\hspace{0.4cm}\includegraphics[width=0.46\textwidth]{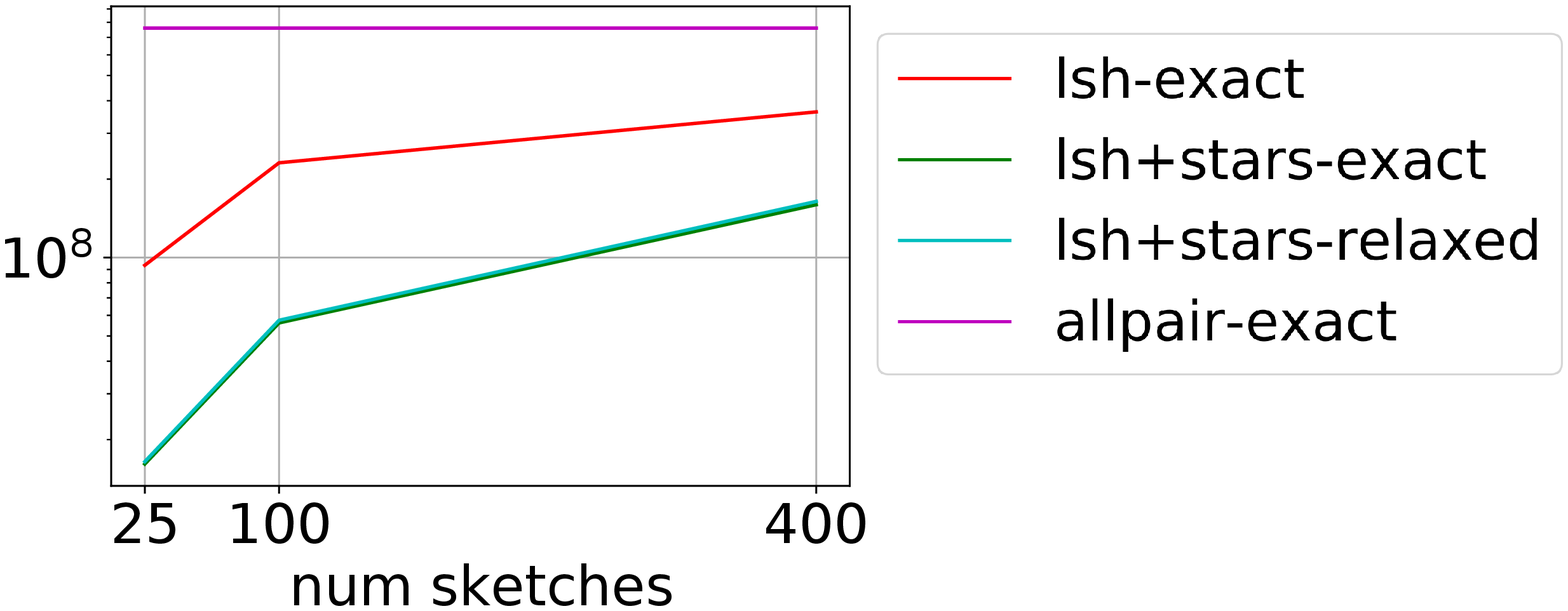}\\
Wikipedia&Amazon2m& \hspace{-2cm} MNIST\\
\end{tabular}
    \caption{\small The number of edges with similarity at least $0.5$ ($0.495$ for relaxed threshold) built by each algorithm.}
    \label{fig:sparsity} 
\end{figure*}

\noindent
\textbf{Coverage of Near(est) Neighbors.}
We evaluate the number of (approximate) near(est) neighbors which can be found for each point in the dataset, and by each algorithm.
We run the bruteforce (AllPair) algorithm for MNIST, Wikipedia and Amazon2m (using mixture similarity) datasets to find all ground truth $100$-nearest neighbors and all ground truth near neighbor points with similarity above $0.5$ for each point.
For the non-Stars LSH based algorithm, we compute the fraction of points with similarity at least $0.5$ that are direct neighbors for each point.
For LSH+Stars, we compute two ratios, the first is the fraction of points with similarity at least $0.5$ that can be found in two hops where each edge has similarity also at least $0.5$, and the second is the same except that the two hop edges can have similarity at least $0.495$ (this aligns with the $1.01$-approximation mentioned in Section~\ref{sec:sortingLSH}). 
For SortingLSH based algorithms, we consider the fraction of exact $100$-nearest neighbors can be found in one hop and two hops for non-Stars and Stars respectively in the graph introduced by the $100$-nearest neighbors.
We also consider the relaxation of finding $1.01$-approximate $100$-nearest neighbors (i.e. $1/\eps = 1.01$).
Note that if we can find more than $100$ approximate $100$-nearest neighbors, we regard the ratio as $1$.
In Figure~\ref{fig:degree_recall}, we report the average of each ratio over all data points.
The graphs built using Stars are able to find more near(est) neighbors in $2$ hops, with fewer edges overall. 
Note that the graphs built by SortingLSH-based algorithms have the same sparsity since we only keep the $250$ closest points for each node (even for SortingLSH+Stars).
The sparsity of graphs built by LSH-based algorithms is presented in Figure~\ref{fig:sparsity}.
 In Figure~\ref{fig:degree_recall}, we show a good recall of finding $1.01$-approximate near(est) neighbors in $2$ hops by our Stars algorithms. In the meanwhile, the number of edges shown in Figure ~\ref{fig:sparsity} is much less than $n^{1.99}$ which is suggested by our theoretical guarantees.
 This phenomenon is observed in all $3$ datasets used in Figure ~\ref{fig:degree_recall},\ref{fig:sparsity}.

\begin{figure*}
   \centering
\begin{tabular}{cc}
\small
\includegraphics[width=0.54\textwidth]{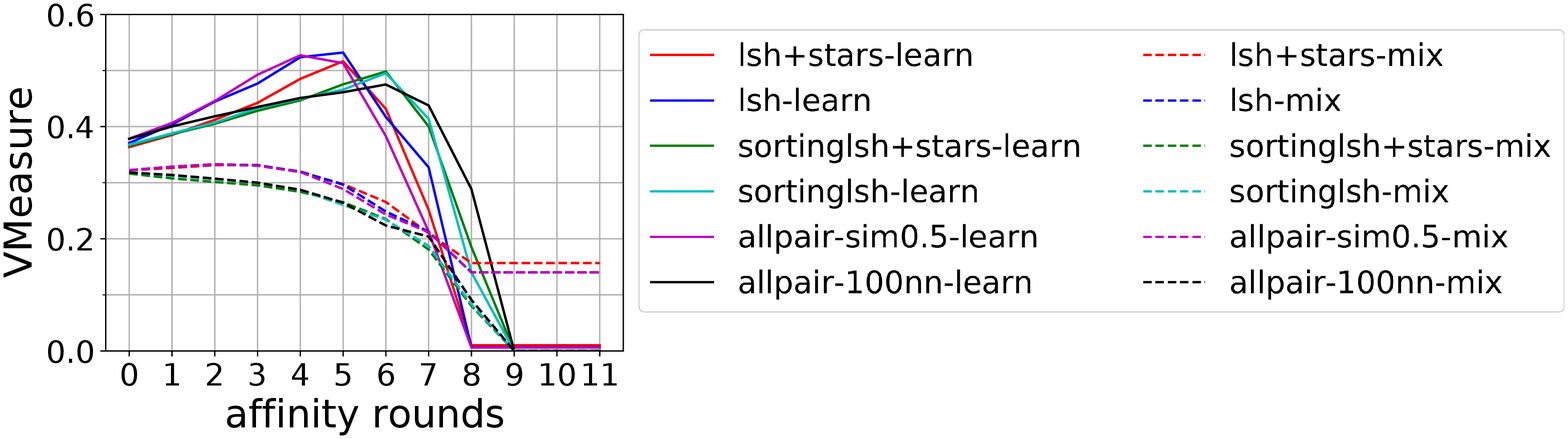}&
\includegraphics[width=0.36\textwidth]{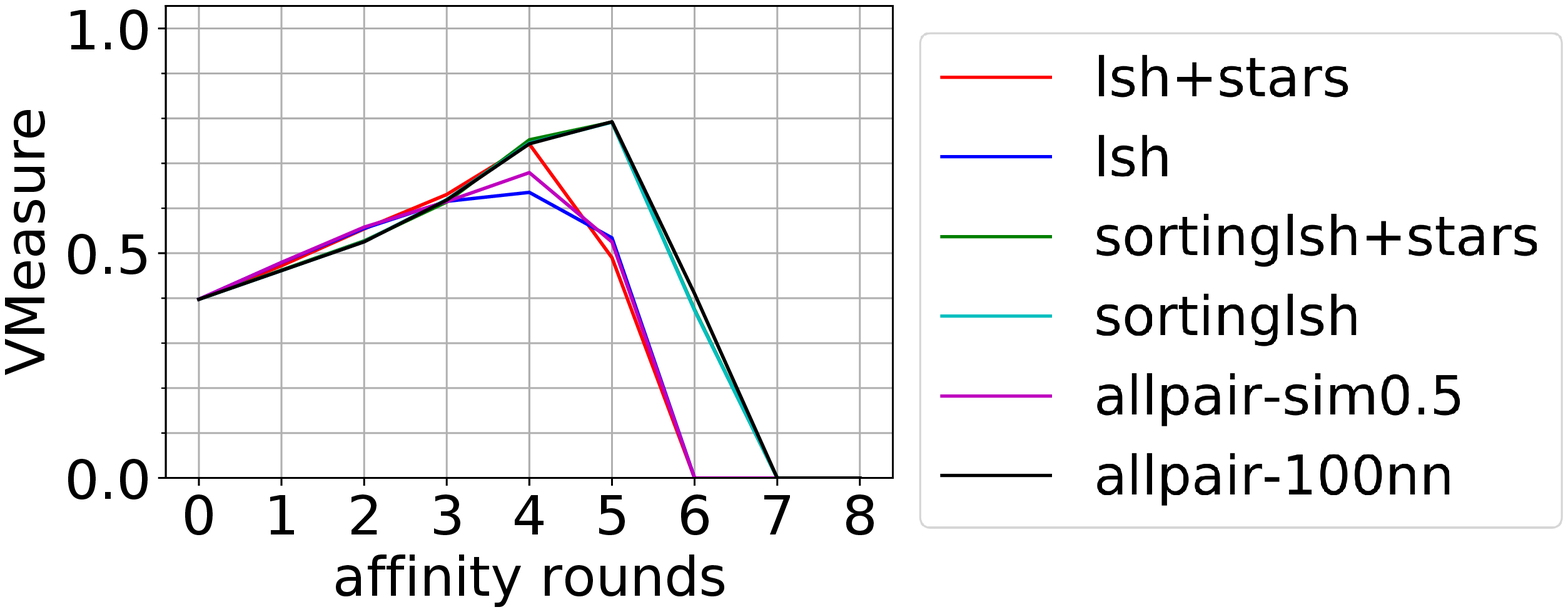}\\
Amazon2m& \hspace{-2cm} MNIST\\
\end{tabular}
    \caption{\small The VMeasure scores of clusterings. The \textit{allpair-100nn} indicates the ground truth $100$-nearest neighbor graph. The \textit{allpair-sim0.5} indicates the ground truth near neighbor graph induced by all edges with similarity at least $0.5$. The suffix \textit{learn} indicates the similarity learned by neural network. The suffix \textit{mix} indicates the mixture similarity of cosine similarity and Jaccard similarity.}
    \label{fig:vmeasure} 
    \vspace{-5mm}
\end{figure*}

\noindent
\textbf{Clustering.}\label{sec:exp_clustering}
The points from MNIST are drawn from $10$ classes, and the points from Amazon2m are drawn from $47$ classes. 
To cluster the graphs, we run average Affinity clustering \cite{bateni2017affinity} on the graphs built by each algorithm.
In particular, for graphs built by LSH-based algorithms, we only keep the edges with similarity at least $0.5$, and for graphs built by SortingLSH-based algorithms, we only keep the closest direct $100$ nodes for each point.
For Amazon2m, we consider both mixture similarity and learning similarity, i.e., we built graphs for each similarity and apply clustering on these graphs.
For all algorithms, we use the number of sketches $R=400$.
We measure the clustering quality via VMeasure score~\cite{rosenberg2007v} which is the harmonic mean between homogeneity score and completeness score of a clustering.
The VMeasure score is in $(0,1)$.
The higher the VMeasure score, the higher the quality of the clusters.
We report VMeasure scores in Figure~\ref{fig:vmeasure}.

%\noindent
%\textbf{Scalability}
%Since the actual running time of each job depends on the number of workers assigned, the traffic of the network and the I/O latency, etc., we instead report the total running time of finding all edges over all workers.
\noindent
\textbf{Effect of the similarity function.} 
As shown in Figure~\ref{fig:vmeasure}, using a similarity function learned by a neural network in the graph building stage indeed helps other downstream tasks such as clustering on the built graphs.
However, a sophisticated similarity function increases the running time of computing the similarity between two points.
In fact, all of our algorithms become $5\times\sim 10\times$ (in terms of the total running time over all workers) slower when using neural network similarity instead of using the mixture of cosine and Jaccard similarity to build the graph for Amazon2m dataset.
Fortunately, the Stars graph building algorithms use significantly fewer comparisons, and are $10\times$ faster than the non-Stars versions (in terms of total running time over all workers).
In addition, they lose only negligible quality for downstream tasks. 
Thus, given the same computational resources, switching to a Stars-based graph building strategy enables us to employ a wider range of similarity functions, and affords us the potential to significantly improve the quality of downstream tasks.

\noindent
\textbf{Experiments on Random10B dataset.}
Since we set the degree threshold to be $250$, each algorithm outputs exactly $2.5$ trillion edges for the Random10B dataset.
However, since the non-Stars-based algorithms make $\sim 10^{14}$ comparisons, more than $95\%$ of these comparisons are redundant.
In contrast, the number of comparisons made by Stars-based algorithms is only $2$-$3\times$ the number of edges obtained.
Overall, the total running time of non-Stars LSH based algorithm is $100\times$ the total running time of LSH+Stars.
Similarly, the total running time of non-Stars SortingLSH algorithm is $10\times$ the total running time of SortingLSH+Stars.
The actual runtime of all Stars-based algorithms finish in $2$ hours, whereas the non-Stars SortingLSH algorithm finishes in $4$ hours, and the non-Stars LSH algorithm finishes in $23$ hours.

%%%%%%%%%%%%%%%%%%%%%%%%%%%%%%%%%%%%%%%%%%%%%%%%%%%%%%%%%%%%
%%%%%%%%%%%%%%%%%%%%%%%%%%%%%%%%%%%%%%%%%%%%%%%%%%%%%%%%%%%%
%Bibliograph
\bibliography{cluster}

%%%%%%%%%%%%%%%%%%%%%%%%%%%%%%%%%%%%%%%%%%%%%%%%%%%%%%%%%%%%
%%%%%%%%%%%%%%%%%%%%%%%%%%%%%%%%%%%%%%%%%%%%%%%%%%%%%%%%%%%%
\section*{Checklist}

\begin{comment}
%%% BEGIN INSTRUCTIONS %%%
The checklist follows the references.  Please
read the checklist guidelines carefully for information on how to answer these
questions.  For each question, change the default \answerTODO{} to \answerYes{},
\answerNo{}, or \answerNA{}.  You are strongly encouraged to include a {\bf
justification to your answer}, either by referencing the appropriate section of
your paper or providing a brief inline description.  For example:
\begin{itemize}
  \item Did you include the license to the code and datasets? \answerYes{See Section~\ref{gen_inst}.}
  \item Did you include the license to the code and datasets? \answerNo{The code and the data are proprietary.}
  \item Did you include the license to the code and datasets? \answerNA{}
\end{itemize}
Please do not modify the questions and only use the provided macros for your
answers.  Note that the Checklist section does not count towards the page
limit.  In your paper, please delete this instructions block and only keep the
Checklist section heading above along with the questions/answers below.
%%% END INSTRUCTIONS %%%
\end{comment}

\begin{enumerate}

\item For all authors...
\begin{enumerate}
  \item Do the main claims made in the abstract and introduction accurately reflect the paper's contributions and scope?
  \answerYes{}
%    \answerTODO{}
  \item Did you describe the limitations of your work? 
    \answerYes{As discussed in Sections \ref{sec:prelims} and \ref{sec:stars}, for our theoretical bounds to hold we require a locality sensitive hash family for the similarity measure in question. }
  \item Did you discuss any potential negative societal impacts of your work?  \answerNA{}
    %\answerTODO{}
  \item Have you read the ethics review guidelines and ensured that your paper conforms to them? \answerYes{}

\end{enumerate}

\item If you are including theoretical results...
\begin{enumerate}
  \item Did you state the full set of assumptions of all theoretical results?  \answerYes{}
    %\answerTODO{}
        \item Did you include complete proofs of all theoretical results?  \answerYes{Full proofs are given in the supplementary of all theoertical statements. }
    %\answerTODO{}
\end{enumerate}

\item If you ran experiments...
\begin{enumerate}
  \item Did you include the code, data, and instructions needed to reproduce the main experimental results (either in the supplemental material or as a URL)?
    \answerYes{}
  \item Did you specify all the training details (e.g., data splits, hyperparameters, how they were chosen)?
    \answerYes{ }
        \item Did you report error bars (e.g., with respect to the random seed after running experiments multiple times)?
    \answerNo{Most our datasets are very large, and there the variance of the runtime and number of comparisons will be quite low due to concentration bounds. }
        \item Did you include the total amount of compute and the type of resources used (e.g., type of GPUs, internal cluster, or cloud provider)?
    \answerYes{}
\end{enumerate}

\item If you are using existing assets (e.g., code, data, models) or curating/releasing new assets...
\begin{enumerate}
  \item If your work uses existing assets, did you cite the creators? 
    \answerYes{ }
  \item Did you mention the license of the assets?
    \answerYes{}
  \item Did you include any new assets either in the supplemental material or as a URL?
    \answerYes{}
  \item Did you discuss whether and how consent was obtained from people whose data you're using/curating?
    \answerNA{}
  \item Did you discuss whether the data you are using/curating contains personally identifiable information or offensive content?
    \answerNA{}
\end{enumerate}

\item If you used crowdsourcing or conducted research with human subjects...
\begin{enumerate}
  \item Did you include the full text of instructions given to participants and screenshots, if applicable?
    \answerNA{}
  \item Did you describe any potential participant risks, with links to Institutional Review Board (IRB) approvals, if applicable?
    \answerNA{}
  \item Did you include the estimated hourly wage paid to participants and the total amount spent on participant compensation?
    \answerNA{}
\end{enumerate}

\end{enumerate}

%%%%%%%%%%%%%%%%%%%%%%%%%%%%%%%%%%%%%%%%%%%%%%%%%%%%%%%%%%%%

\newpage

\appendix

\section*{Appendix}

\section{Connected Components and Single-linkage Clustering}\label{sec:connectedComponents}
We now demonstrate that given a $(r_1,r_2)$-two-hop spanner for $(P,\mu)$, one can obtain an approximate solutions to single-linkage clustering by varing $(r_1,r_2)$. 

\begin{observation}
	For $c\geq 1$ and $r>0$, if two points are in the same connected components of $r$-threshold graph, they must be in the same connected components of $(r/c,r)$-two-hop spanner. Furthermore, if two points are in the same connected components of $(r/c,r)$-two-hop spanner, they must be in the same connected components of $r/c$-threshold graph.
\end{observation}
\begin{corollary}
	For $c\geq 1$ and $r>0$, the number of connected components of $(r/c,r)$-two-hop spanner is at least the number of connected components of  $r/c$-threshold graph, and is at most the number of connected components of $r$-threshold graph. 
\end{corollary}

Given a parameter $k\geq 1$, the goal of the single-linkage clustering is to partition the input data $P$ into $k$ clusters $C_1,C_2,\cdots,C_k$ such that the maximum similarity between points in separate clusters, namely the quantity  minimized $\max_{i \neq j} \max_{p\in C_i,q\in C_j} \mu(p,q)$
is minimized
Let 
\[\OPT_k = \min_{C_1,C_2,\cdots,C_k} \max_{i \neq j} \max_{p\in C_i,q\in C_j} \dist(p,q) \] 
i.e., $\OPT_k$ denotes the optimal cost of $k$-single-linkage clustering.
Let $r$ be in the range $[\OPT_{k},\OPT_{k+1})$.
Then by the definition of $\OPT$, it is easy to verify that the connected components of $r$-threshold graph yield the optimal $k$-single-linkage clustering.
In the following, we show that we can obtain an approximate $k$-single-linkage clustering via a two-hop spanner with the appropriate parameters.

%\rajesh{TODO: replace this with new notation}
\begin{theorem}
	Let $c\geq 1$ and $r<\OPT_{k}/c$, where $\OPT_k$ is optimal cost of $k$-single-linkage clustering on $(P,\mu)$.  Then any $(r/c,r)$-2-hop spanner $G$ has at least $k$ connected components.
	Furthermore, for any two connected components $C,C'$ of $G$, $\min_{x\in C,y\in C'}\mu(x,y)\geq r$, and the number of connected components of $G$ is at least $k$.
\end{theorem}
\begin{proof}
	By the definition of $(r/c,r)$-$2$-hop spanner, if $\dist(p,q)> r$, there is no edge between $p$ and $q$.
	Let us show that the connected components of $G$ is at least $k$.
	We prove it by contradiction.
	Let $C_1,C_2,\cdots,C_k$ be the optimal $k$-single-linkage clusters.
	Suppose $G$ has at most $k-1$ clusters.
	By pigeonhole principle, we can find $p,q$ which are in the same connected components but $p\in C_i,q\in C_j$ for some $i\not = j$.
	Since $p,q$ are in the same connected components, there must be a path between $p$ and $q$ in $G$.
	We can find an edge $(p',q')$ on the path such that $p'\in C'_{i'},q'\in C'_{j'}$ for some $i'\not = j'$.
	Since $C_1,C_2,\cdots,C_k$ is the optimal $k$-single-linkage clustering, we have $\dist(p',q')\geq \OPT_k>c\dot r$ which leads to a contradiction.
	
	Next let us show that for any two connected components $C,C'$ of $G$, $\min_{x\in C,y\in C'}\dist(x,y)> r$. 
	Suppose there are two points $p,q$ satisfying $\dist(p,q)\leq r$.
	By the definition of $2$-hop spanner, there is a path between $p$ and $q$ with at most $2$ hops.
	Therefore, for any two points that are in different connected components, their distance is at least $r$.
\end{proof}

A simple corollary of the above theorem is that we can easily obtain a $k$-single-linkage clustering solution with cost at least $r$ by arbitrarily merging connected components to reduce the number of clusters to $k$.

\section{Proofs Omitted from Section \ref{sec:stars}}
In this section, we fill in missing details and proofs from statements and theorems made iun Section \ref{sec:stars}. Firstly, we prove Theorem \ref{thm:lshMain} below. Afterwards, we first formalize the discussion on the existence of certain $(r_1,r_2,\rho)$-sensitive families for the Jaccard and Cosine similarity measures.
\subsection{Proof of Theorem \ref{thm:lshMain}}

\begin{proof}
First, note that by construction, the algorithm \hyperlink{twohopLSH}{Stars 1} never creates an edge between two points $x,y \in P$ with $\mu(x,y)< r_1$. Thus, the first condition for a $(r_1,r_2)$-two hop spanner (Definition \ref{def:twohop}) is satisfied deterministically. To show the second condition, fix any point $p \in P$, and fix any $q \in P$ with $\mu(p,q) > r_2$. We show that $q \in \cN_2(p)$ with high probability, which will complete the proof. 

Firstly, by definition of a $(r_1,r_2,\rho)$-sensitive family, the probability that $p,z$ collide in a hash bucket $B_u$ when $\mu(p,z) < r_1$ is at most $1/n^4$. Thus, the probability that any pair $p,z$ collide with $\mu(p,z) < r_1$ in a single repetition is at most $1/n^2$, and by a union bound at most $R/n^2 < 1/(2n)$ over all $R = c_1 n^\rho \log n$ repetitions, where $c_1$ is a sufficiently large constant. Call this event $\cE$, and condition on it now. Then, again by definition of a $(r_1,r_2,\rho)$-sensitive family, we have that $p,q$ collide in a hash bucket $B_u$ with probability at least $n^{-\rho}$. Thus, by repeating the hashing $R = c_1 n^\rho \log n$ times, it follows $p,q$ collide in at least one repetition with probability at least $1-1/n^4$. In this repetition, by event $\cE$, we have that all pairs of points which are contained in $B_u$ have pairwise similarity at least $r_1$. In particular, if $x \sim B_u$ is the uniformly sampled leader in algorithm \hyperlink{twohopLSH}{Stars 1}, we have that $\mu(x,p) > r_1$ and $\mu(x,q) > r_1$, thus the edges $(x,p),(x,q)$ will be added to the graph with probability at least $1-\Pr[\cE] - 1/n^{-4} > 1-1/n$, which completes the proof. 
\end{proof}

We now formalize the claim made after the statement of Theorem \ref{thm:lshMain} about the existence of  $(r_1,r_2,\rho)$-sensitive families for the Jaccard and cosine (angular) similarity measures, where the latter is given by $\mu(x,y) = 1-\theta_{x,y}$. The proof of the following proposition follows immedietly from the proof of Proposition \ref{prop:simhashandJaccard}, by simply replacing the threshold $\theta_k(p)$ with the fixed threshold $\alpha$, and setting $M = \infty$ (to avoid the second case in Proposition \ref{prop:simhashandJaccard}, which is not needed for the following claim).

\begin{proposition}
Let $\mu$ be either the angular or Jaccard similarity measure on a dataset $P$. Fix any $\eps,\alpha \in (0,1)$. Then there exists a $(1-\eps^{-1}\alpha, 1-\alpha, O(\eps))$-sensitive hash family for $\mu$.
\end{proposition}

\subsection{Proof of Propositions \ref{prop:simhashandJaccard}}
We now provide the proof of Proposition \ref{prop:simhashandJaccard}.

\begin{proposition}[Approximate Nearest Neighbors for Angular Similarity]\label{prop:simhash}
Let $P \subset \R^d$ be a subset, and let $\mu(x,y) = 1- \theta_{x,y}$ where $\theta_{x,y}$ is the angle between $x,y \in P$ (normalized so that $\theta_{x,y} \in [0,1]$). Let $\cH$ be the SimHash family, where $h \sim \cH$ is selected by first drawing $z \in \R^n$ uniformly from the unit sphere, and setting $h(x) = \textsc{sign}(\langle x,z \rangle)$. Then for any $\eps \in (0,1)$, integer $M \geq 1$,  and $p \in P$, let $\theta_k(p) = 1- \tau_k(p)$ be the normalized $k$-th closest angle to the point $p$. Then the family $\cA = \{\cA_p\}_{p \in P}$, where
\[  \cA_p= \left\{ x \in P \; | \; \mu(p,x) \geq \min\left\{ 1- \frac{\theta_k(p)}{\eps}, 1-\frac{1}{M} \right\} \right\} \]
is a $(k, O(\eps), 4 M  \log n)$-ANN family with respect to $\cH$.
\end{proposition}
\begin{proof}
   Fix any $p \in P$, and set $\ell = \ell_p = \min\{ \frac{4 \epsilon \log n}{\theta_k(p)}, 4 M \log n\}$. In the first case, suppose $\ell = \frac{4 \epsilon \log n}{\theta_k(p)} \leq  4 M \log n$. 
   Then note that, for any $x \in N_k(p)$, the probability that $(h_1(p),\dots,h_\ell(p)) =(h_1(x),\dots,h_\ell(x))$ is at least 
   \[(1-\theta_k(p)^{\ell} = (1 - \theta_k(p))^{\frac{4 \epsilon \log n}{\theta_k(p)}}\geq  n^{-O(\eps)}\]
   Similarly, for $x \notin \cA_p$, we have $\mu(p,x) <  1- \frac{\theta_k(p)}{\eps}$,  so the probability that $(h_1(p),\dots,h_\ell(p)) =(h_1(x),\dots,h_\ell(x))$ is at most 
   
   \[\left( 1- \frac{\theta_k(p)}{\eps}\right)^{\frac{4 \epsilon \log n}{\theta_k(p)}} \leq 1/n^4 \]
   where we used the inequality that $(1-x)^{n/x}  \leq \left(\frac{1}{2}\right)^n$ for any $x \in (0,1]$ and $n \geq 1$. 
   
   Next, suppose that $\ell_p = 4M\log n <  \frac{4 \epsilon \log n}{\theta_k(p)}$. In this case, since $\ell_p$ is only smaller than the threshold used above, we still have that $h_1(p),\dots,h_\ell(p) =h_1(x),\dots,h_\ell(x)$ with probability at least $n^{-O(\eps)}$ for any $x \in N_k(p)$. Next, for any $x \notin \cA_p$,  the probability that $(h_1(p),\dots,h_\ell(p)) =(h_1(x),\dots,h_\ell(x))$ is at most 
   \[\left(1-\frac{1}{M }\right)^{\ell_p} =  \left( 1- \frac{1}{M}\right)^{4 M \log n} \leq 1/n^4 \]
   which completes the proof. 
   
\end{proof}

\begin{proposition}[Approximate Nearest Neighbors for Jaccard Similarity]\label{prop:jaccard}
Let $P \subset 2^{\cU}$ be a set of subsets of a universe $\cU$, and let $\mu(x,y) = \frac{|x \cap y|}{|x \cup y|}$ be the Jaccard Similarity between $x,y \subset \cU$. Let $\cH$ be the MinHash family, where $h \sim \cH$ is selected by first drawing a random number $n_u \sim [0,1]$ for each $u \in \cH$, and setting $h(x)= \min_{u \in x} n_u$.  Then for any $\eps \in (0,1)$, integer $M \geq 1$,  and $p \in P$, let $s_k(p) = 1-\tau_k(p)$ be the $k$-th smallest Jaccard distance to the point $p$. Then the family $\cA = \{\cA_p\}_{p \in P}$, where
\[  \cA_p= \left\{ x \in P \; | \; \mu(p,x) \geq \min\left\{1- \frac{s_k(p)}{\eps}, 1-\frac{1}{M} \right\} \right\} \]
is a $(k, O(\eps),4 M \log n)$-ANN family with respect to $\cH$.
\end{proposition}
\begin{proof}
    The proof of Proposition \ref{prop:jaccard} is identical to that of Proposition \ref{prop:simhash}, by simply using that fact that $\Pr[h(x) =h(y)] =\frac{|x \cap y|}{|x \cup y|}$ for the case of MinHash. 
\end{proof}

\subsection{Proof of Theorem \ref{thm:sortingLSHMain}}

\begin{proof} We begin by demonstrating that the first condition holds, by demonstrating that our algorithm finds enough edges so that $\left| \mathcal{N}^2_{\hat{G}_p}(p) \right|  \geq (1-\delta)k$. 
  In what follows, fix any vertex $v \in P$. We run \hyperlink{twohopSortingLSH}{Stars 2} with number of iterations $R = O(\log n \cdot n^\rho)$. For each $i \in [R]$, let $H_i = (h_{i,1},h_{i,2},\dots,h_{i,M})$ be the hash functions drawn on repetition $i$, and for $t \in [M]$, let $H_i^t = (h_{i,1},h_{i,2},\dots,h_{i,t})$ be the first $t$ hash functions drawn in repetition $i$. 

For each repetition $i \in [R]$, call $i$ balanced for the point $p$ if $p$ is at distance $W/4$ from either boundary of the block $B_j$ which contains it if that block is not the first or last block, and if $B_j$ is the first or last block then $i$ is well balanced if it is distance at least $W/4$ from the only block adjacent to $B_j$. It is easy to see that, in either case, a repetition $i$ is balanced for $p$ with probability at least $1/2$, taken only over the random choice of $r \in \{W/2,\dots,W\}$. By Chernoff bounds, it follows that at least $R/3$ repetitions will be balanced for $i$ with probability at least $1-1/n^3$, which we condition on now. 
In what follows, we restrict our attention only to the balanced repetitions, and show that $\left| \mathcal{N}^2_{\hat{G}_p}(p) \right|  \geq (1-\delta)k $  holds even restricted to just the edges found in those repetitions. 

By definition of the set $\cA_p$ and a union bound over $n$ points and at most $R \leq n$ repetitions, for each (balanced) repetition $i$ there exists an $\ell = \ell_{p} \leq M$ such that simultaneously for all $x \notin \cA_p$, we have $H_i^\ell(p)  \neq H_i^\ell(x)$ with probability at least $1-1/n^2$. Call the event that this holds $E_p$, and condition on it for the remainder of the proof. Next, for each $x \in N_k(p)$, note that $H_i^\ell(p)  = H_i^\ell(x)$ with probability at least $n^{-\rho}$. Moreover, the above analysis on the probability that a repetition $i$ only took randomness over the choice of the random $r \in \{W/2,\dots,W\}$, it follows that the event that both $i$ is balanced and $H_i^\ell(p)  = H_i^\ell(x)$ is at least $n^{-\rho}/2$. Thus, since we take $R = c_1 n^{\rho} \log n$ repetitions, taking $c_1 > 50$ with probability at least $1-n^{-10}$ we will have  $H_i^\ell(p)  = H_i^\ell(x)$ for at least one balanced repetition $i$. Call this event $E_2$ and coniditon on it for the remainder of the proof. 
We now split into cases based on the two branches of the if-statement in \hyperlink{twohopSortingLSH}{Stars 2}.  

\paragraph{Case 1:} $k >  R/\delta$. For any balanced repetition $i$, we say that $i$ is bad if 
\[\left|\{x \in N_k(p) \; | \; H_i^{\ell_p}(p) = H_i^{\ell_p}(x)\}\right| \leq \frac{\delta k}{R}  = \frac{ \delta k }{ c_1 n^{\rho} \log n}  \]
We call a repetition good if it is not bad. Now for any point $x \in N_k(p)$, we say that $x$ is good if there exists a balanced repetition $i$ such that $H_i^{\ell_p}(p) = H_i^{\ell_p}(x)$ and $i$ is good. For such a good repetition $i$, we say that $x$ is good at $i$. 
In what follows, we fix any point $x \in N_k(p)$.  We now prove the following claim.
\begin{claim}
If  $x \in N_k(p)$ is good at $i$, then with probability at least $1-n^{-10}$, either $x$ is added to $\mathcal{N}^2_{\hat{G}_p}(p)$ on repetition $i$, or we add at least $k$ unique points from $\cA_p$ to $\mathcal{N}^2_{\hat{G}_p}(p)$ on repetition $i$. 
\end{claim}
\begin{proof}
  
To see this, for any point $y$, define $\sigma_y \in \{1,2,\dots,n\}$ be the ranking of $y$ in the lexicographical ordering induced by the set of hash functions $H^i$ corresponding to the $i$-th repetition. Let $t_{-} \leq t_p$ be the smallest possible index such that the point $y$ with $\sigma_y = t_{-}$ satisfies  $H_i^{\ell_p}(p) = H_i^{\ell_p}(y)$, and similarly define the point $t_{+} \geq t_p$ as the largest such index. Note that, by conditioning on $E_p$, for every $y$ such that $\sigma_y \in [t_{-},t_+]$, we have $y \in \cA_p$. 

First suppose that either $t_{-} < t_p - W/4$ or $t_{+} > t_p + W/4$. Then since repetition $i$ is balanced, it follows that there are at least $W/4 > k$ points $y \in \cA_p$ which are contained in the same block as $p$. Since we sample $s = c_2 \delta^{-1} n^{ \rho} \log^2 n$ random leaders in this block, taking $c_2>100$, with probability larger than $1-1/n^{10}$ we will sample at least one $y \in \cA_p$ which in that window. Then for any other $z \in \cA_p$ which also fell in the same block, we create a path $(z,y),(y,p)$ of length two, each of which consists of only edges between points in $\cA_p \cup \{p\}$. Thus, in this case we add at least $W/4 > k$ unique points from $\cA_p$ to $\mathcal{N}^2_{\hat{G}_p}(p)$. 

Otherwise, we have that $t_p - W/4 < t_{-} < t_p < t_{+} < t_p + W/4$. Again, since repetition $i$ is balanced, it follows that every $y$ with  $\sigma_y \in [t_{-},t_+]$ falls in the same window as $p$. Since we know that repetition $i$ is good, it follows that $t_{+} - t_{-} > \frac{\delta k}{c_1 n^\rho \log n}$. Thus, again by sampling $s = c_2 \delta^{-1} n^{ \rho} \log^2 n$ leaders for $c_2$ larger than some constant times $c_1$, again with probability larger than $1-1/n^{10}$ we will sample at least one $y$ such that $\sigma_y \in [t_{-},t_+]$ as a leader, implying that $y \in \cA_p$. Conditioned on this, since $x$ must also fall in the same window, we add the edges $(x,y),(y,p)$ to $\hat{G}_p$, thereby adding $x$ is added to $\mathcal{N}^2_{\hat{G}_p}(p)$, which completes the proof of the claim.
\end{proof}

By the above claim and a union bound, with probability at least $1-n^{-9}$ it holds that for every $x \in N_k(p)$, if $x$ is good then either we add $x$ to $\mathcal{N}^2_{\hat{G}_p}(p)$ or we already have that $\left| \mathcal{N}^2_{\hat{G}_p}(p) \right|  \geq k$. To complete the proof for Case 1, it suffices to show that the number of good points  $x \in N_k(p)$ is at least $(1-\delta)\cdot k$. To see this, note that if point $x$ is not good, then for every balanced repetition $i$ such that $H_i^{\ell_p}(p) = H_i^{\ell_p}(x)$ (for which there is at least one due to conditioning on $E_2$), we have $\left|\{x \in N_k(p) \; | \; H_i^{\ell_p}(p) = H_i^{\ell_p}(x)\}\right| \leq \delta k / R$. For each bad point, we charge it's ``badness'' to the first balanced repetition $i$ for which $H_i^{\ell_p}(p) = H_i^{\ell_p}(x)$ and such that $i$ was bad. Since each bad repetition can be charged for at most  $\delta k / R$ bad points, it follows that there can be  at most $R \cdot \delta k / R < \delta k$ bad points in $N_k(p)$, and therefore the number of points good points  $x \in N_k(p)$ is at least $(1-\delta)\cdot k$, which completes the proof of Case 1.

\paragraph{Case 2:} $k \leq  R/\delta$. In this case, we will demonstrate a stronger claim: for any $x \in N_k(p)$, and any balanced repetition $i$ such $H_i^{\ell_p}(p) = H_i^{\ell_p}(x)$  (for which there is at least one due to conditioning on $E_2$), either either the edge $(x,p)$ is added $\hat{G}_p$ on repetition $i$, or for at least $k$ unique points $y \in \cA_p$ we add the edge $(y,p)$ points to $\hat{G}_p$ on repetition $i$. Note that, given this claim, the stronger fact that $\left| \mathcal{N}^1_{\hat{G}_p}(p) \right|  \geq k$ in Case 2 follows immediately. 

To see this, fix any $x \in N_k(p)$, and any balanced repetition $i$ such $H_i^{\ell_p}(p) = H_i^{\ell_p}(x)$. First suppose that $|\sigma_x - \sigma_p| \geq W/4$, WLOG we have $\sigma_x \geq \sigma_p + W/4$ (note that for the edge cases where $p$ is in the first or last block, there may be only one choice between $\sigma_x < \sigma_p$ or $\sigma_x > \sigma_p$ in this setting). Since $H_i^{\ell_p}(p) = H_i^{\ell_p}(x)$, it follows that every $y$ with $\sigma_x < \sigma_y < \sigma_x + W/4$ also satisfies $H_i^{\ell_p}(p) = H_i^{\ell_p}(y)$, and therefore $y \in \cA_p$. Since $i$ is balanced, it follows that all such $y$ are in the same block as $p$, and therefore we created at least $W/4 > k$ such edges $(y,p)$ for points $y \in \cA_p$.
In the second case, assume $|\sigma_x - \sigma_p| \geq W/4$. In this case, by the balancedness of repetition $i$, it follows that $x,p$ are in the same block, and therefore we connect $(x,p)$ on this repetition, which completes the proof for case $2$.

\paragraph{Bounding The Number of Edges.} In case $1$, for each of the $R$ repetitions, we create at most $2ns$ edges, thus the total number of edges is $O(Rns) = \tilde{O}(\delta^{-1}n^{1 + 2 \rho})$. In Case 2, each vertex has degree at most $W$, so we create at most $Wn = O(k n) = O(n R/ \delta)$ edges, thus the total number of edges is $RWn = \tilde{O}(\delta^{-1}n^{1+2 \rho})$ as desired. 

\paragraph{Bounding The Runtime.} Note that evaluating the $M$ hash functions and sorting points lexicographically based on those hash functions requires $O(nM \log n + nM \cdot \Run(\cH))$ time. Thus, the overall runtime to construct all blocks used by the algorithm over $R$ repetitions is $\tilde{O}(RnM \cdot \Run(\cH)) = \tilde{O}(RnM \cdot \Run(\cH))$. Once the blocks are complete, the remaining runtime is within a constant of $|E(G)|$, which is $\tilde{O}(\delta^{-1}n^{1+2 \rho})$ as above, which completes the proof.
\end{proof}

\section{Additional Implementation Details}\label{sec:additional_impl_details}

\subsection{SortingLSH at Scale}
The SortingLSH algorithm involves computing $R$ sketches per point, then sorting the $nR$ total sketches before subdividing into contiguous buckets with window size $W$. To achieve this at billion-point scales, we leverage the TeraSort algorithm \cite{o2008terabyte}.

\subsection{Training a Pairwise Similarity Model}
Following \cite{halcrow2020grale}, we train a similarity model on the task of distinguishing pairs of nodes which are in the same category from pairs of nodes which are in different categories.
The model architecture uses shared-weight embedding towers
to learn a symmetric representation of the node-level features,
which are converted into a pairwise embedding using the Hadamard product.
This embedding is concatenated with additional pairwise features, then used as input for a feed-forward neural net that generates the final binary prediction.
The unthresholded scalar output of this final network
provides a similarity score for any pair of nodes.

\section{Additional Experiments}\label{sec:additional_exp}
%\textcolor{red}{Placeholder}

\subsection{Additional Information of Datasets}
\begin{itemize}
    \item Each data item of MNIST is a $28\times 28$ image of a hand written digit. 
    There are total $10$ different digits and thus MNIST has $10$ ground truth clusters.
    \item Each data item of Wikipedia is an article. 
    The feature of a data item is the set of words appeared in the article together with the frequency of appearance of each word.
    \item Amazon2m consists of products available on the Amazon web site.
    Each product has a $100$-dimensional real valued feature vector and co-purchase relationships with other products. 
    The products are in $47$ categories and thus it has $47$ ground truth clusters.
    \item Random1B and Random10B are generated from mixture of Gaussians.
    Each data point has $100$ dimensions. 
    The number of modes is $100$ where the $i$-th mode is a Gassian distribution with mean $(0,0,\dots,0,1,0,\cdots,0)$ (the $i$-th entry is $1$), and standard deviation $0.1$ of each entry. 
    Each data point is drawn uniformly at random from a mode.
\end{itemize}

\subsection{Detailed Sketching Parameters}
We set number of sketches $R=25,100,400$ for all LSH and SortingLSH based algorithms.
For Stars, if not specified otherwise, we set number of leaders $s=25$. 

For SortingLSH-based algorithms: we set window size $W=250$, sketching dimension $M=30$, and the maximum allowed bucket size to be $20000$ for all experiments.
%For each similarity measure studied, we use a corresponding LSH function to build the graph.

%For MNIST, Random1B and Random10B datasets, %we study the cosine similarity between float vector features, and thus we employ SimHash for them. 
%For LSH-based algorithms, 
For LSH-based algorithms:
\begin{itemize}
\item
For non-Stars version of LSH based algorithm, we set the maximum allowed bucket size to be $1000$.
For Stars version of LSH based algorithm, we set the maximum allowed bucket size to be $10000$.
\item
We employ SimHash of sketching dimension $M=12$ for MNIST and $M=16$ for Random1B and Random10B.
\item
For the Wikipedia dataset, %we study the weighted Jaccard similarity between two sets of strings with weights, and therefore 
we use the weighted Minhash %LSH to build the graph. %and the LSH based algorithms use 
with sketching dimension $M=3$.

\item 
For the Amazon2m dataset, %we study two different similarity measures: %one is based on a mixture of cosine similarity and Jaccard similarities, and the other one is a neural network trained on top of 
%(1) %The first similarity measure is computed by $0.8\times$ the cosine similarity between the float vector features $+$ $0.2\times$ the Jaccard similarity between two sets of strings. 
%a mixture of cosine similarity and Jaccard similarity, and
%Thus, we use a corresponding mixture of SimHash and MinHash to build the graph.
%(2) a neural network where the training set of candidate pairs are generated by SimHash over float vector features and MinHash over sets of strings~\cite{halcrow2020grale}. 
%In summary, in both cases, we use a mixture of SimHash and MinHash for graph building.
%The LSH based algorithms use sketching dimension $M=12$ for the mixture of SimHash and MinHash.
%See Appendix~\ref{sec:additional_exp} for more detailed sketching setups.
we use a mixture of SimHash and MinHash with sketching dimension $M=12$, i.e., randomly select each bit of hash value generated from SimHash or MinHash (it is easy to verify that the mixture of SimHash and MinHash satisfies Definition~\ref{def:LSH} for the similarity which is a mixture of cosine similarity and Jaccard similarity).
\end{itemize}

\subsection{Similarity for Amazon2m Learned by Neural Network}
The embedding tower for each node combines the product embedding and a 1-hot encoded vector of co-purchased products, hashed to 1000 buckets.
The Hadamard product of these embeddings is concatenated with:
the cosine similarity of the product embeddings,
an indicator variable conditioned on two products being copurchased,
and the Jaccard similarity of the two products' copurchase sets.
The embedding towers use two hidden layers of size 100 with ReLU activations \cite{Nair2010RectifiedLU}.
The final prediction is made using the concatenated pairwise features and embedding Hadamard product with another MLP having two hidden layers of size 100 and ReLU activations again.
The model is trained on all pairs which fall into an LSH bucket
from a sample of 500,000 nodes from the original dataset
and achieved an AUC of 0.92 on a holdout set from a distinct sample of nodes for the same-category task.

\subsection{Experiments of Different Number of Leaders}
In this section, we study the impact of number of leaders.
In particular, we fix the number of sketches $R=400$ and evaluate the number of leaders $s=1,5,10,25$.
\paragraph{Number of comparisons.} In Figure~\ref{fig:num_comparisons_different_leaders}, we show the number of pairwise similarity comparisons of each algorithm on MNIST, Wikipedia and Amazon2m datasets.
As we can observe, Stars yields $\sim 100$-fold improvement in number of comparisons over the other algorithms when the number of leaders $s$ is chosen to be $1$ or $5$.

\begin{figure*}[h]
    \centering
         \includegraphics[width=\textwidth]{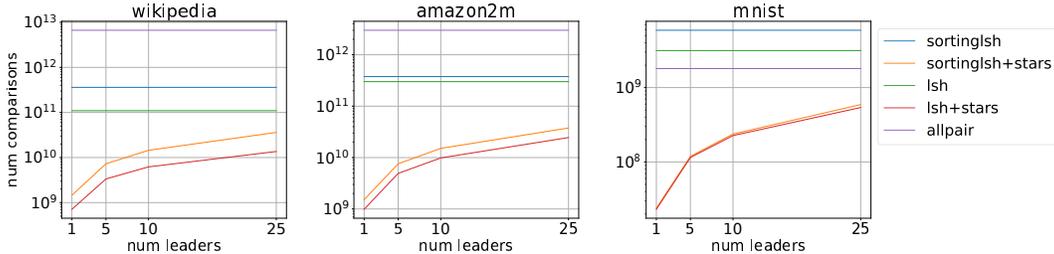}
     \vspace{-2mm}
        \caption{\small Number of comparisons of each algorithm on Wikipedia, Amazon2m and MNIST datasets. If we choose the number of leaders $s$ to be less than $25$, we can get better than $\sim10$-fold improvement in the number of comparisons. In particular, if we choose $s=1$ or $s=5$, we have $\sim 100$-fold improvement in the number of comparisons.}
        \label{fig:num_comparisons_different_leaders}
\end{figure*}

\paragraph{Coverage of Near(est) Neighbors.} 
We evaluate the number of (approximate) near(est) neighbors which can be found for each point in Wikipedia, Amazon2m and MNIST, and by each algorithm with different number of leaders.
For all algorithms, the number of sketches $R=400$.
The evaluation metric is the same as Section~\ref{sec:experiments}.

In Figure~\ref{fig:degree_recall_different_leaders}, we report the average of each ratio over all data points.
When we choose more leaders, we obtain better recall.
%As shown in the figure, our algorithms are robust to the number of leaders.

\begin{figure*}
   \centering
\begin{tabular}{ccc}
\small
\includegraphics[width=\textwidth]{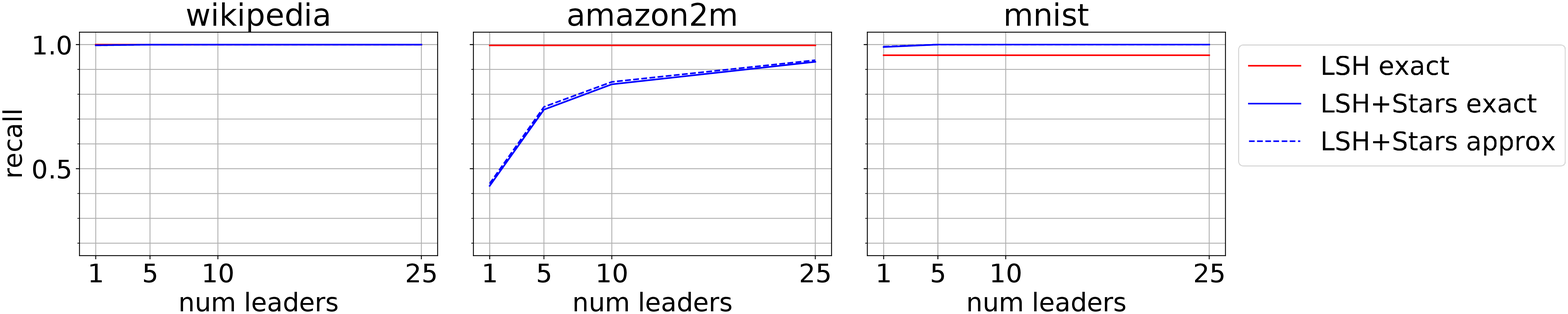}\\
\hspace{5mm}\includegraphics[width=\textwidth]{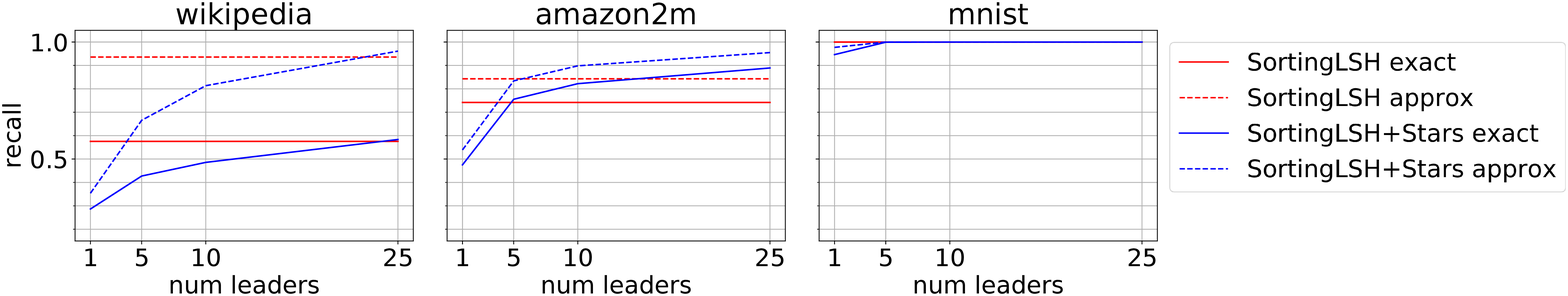}
\end{tabular}
    \caption{\small The recall of found near(est) neighbors when using different number of leaders.}
    \label{fig:degree_recall_different_leaders} 
    \vspace{-5mm}
\end{figure*}

We obtain sparser graphs when we use smaller number of leaders.
Again, note that the graphs built by SortingLSH-based algorithms have the same sparsity since we only keep the $250$ closest points for each node (even for SortingLSH+Stars).
The sparsity of graphs built by LSH-based algorithms is presented in Figure~\ref{fig:sparsity_different_leaders}.

\begin{figure*}[h]
   \centering
\begin{tabular}{ccc}
\small
\includegraphics[width=0.28\textwidth]{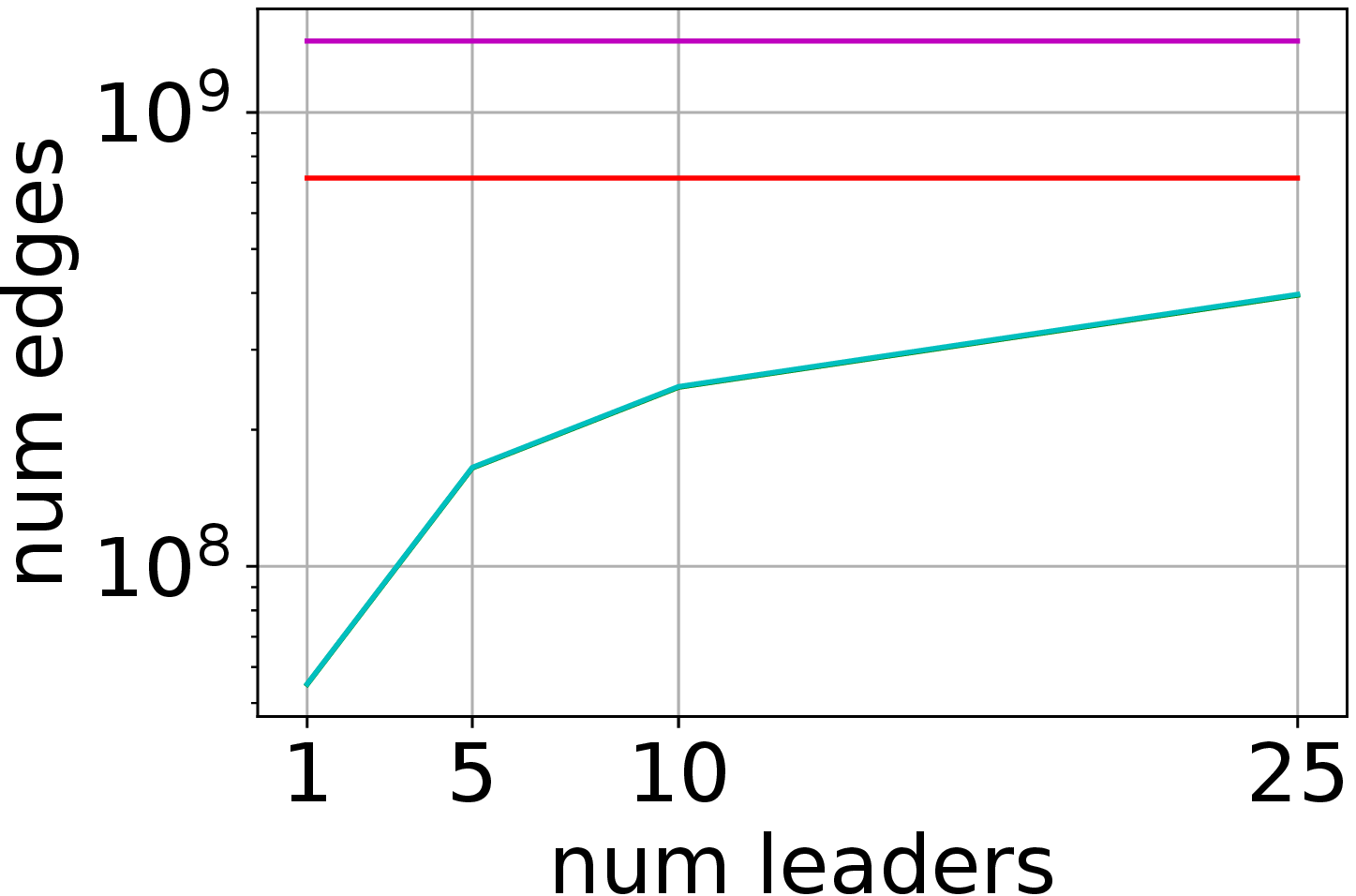}&
\includegraphics[width=0.26\textwidth]{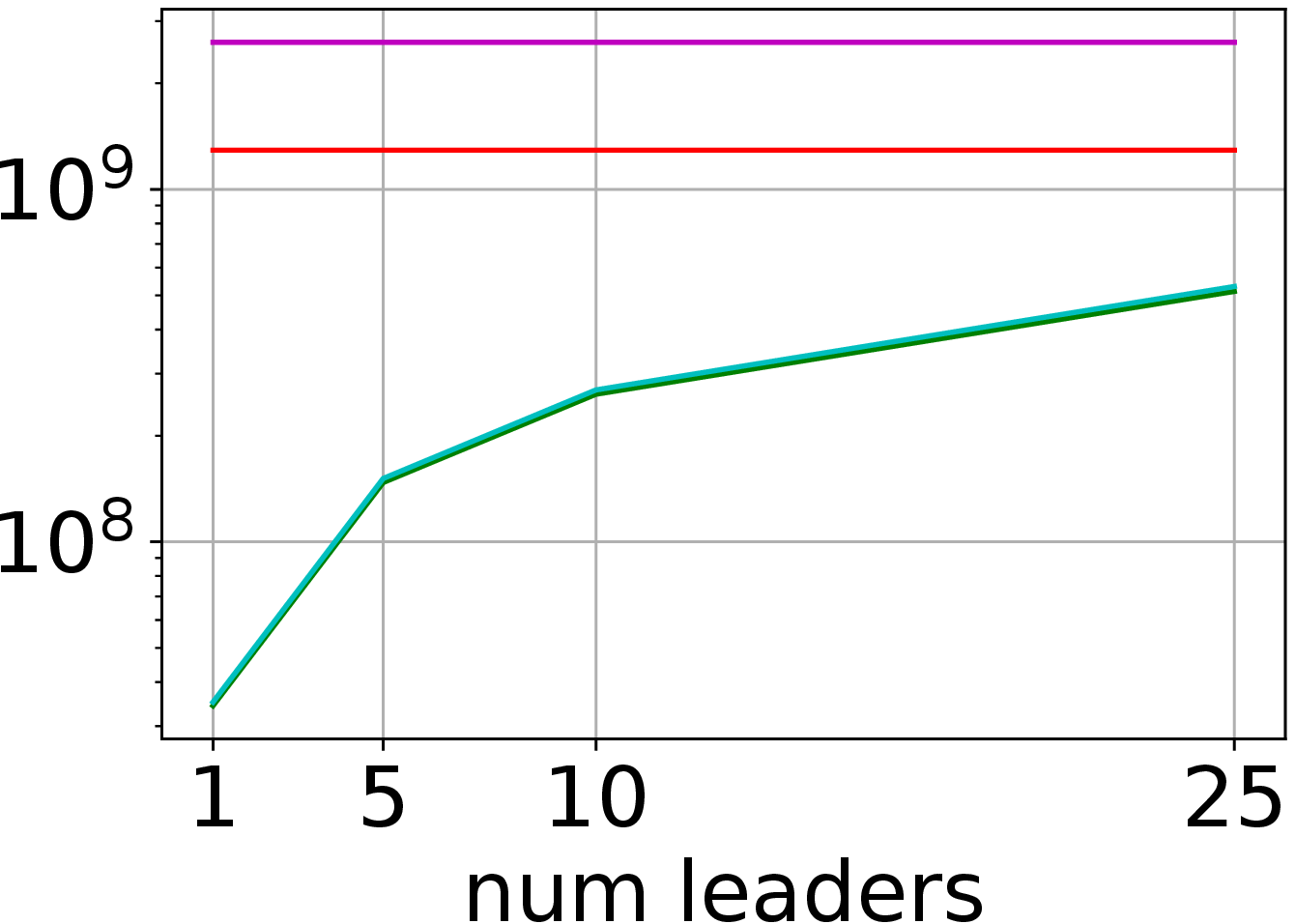}&
\hspace{0.4cm}\includegraphics[width=0.46\textwidth]{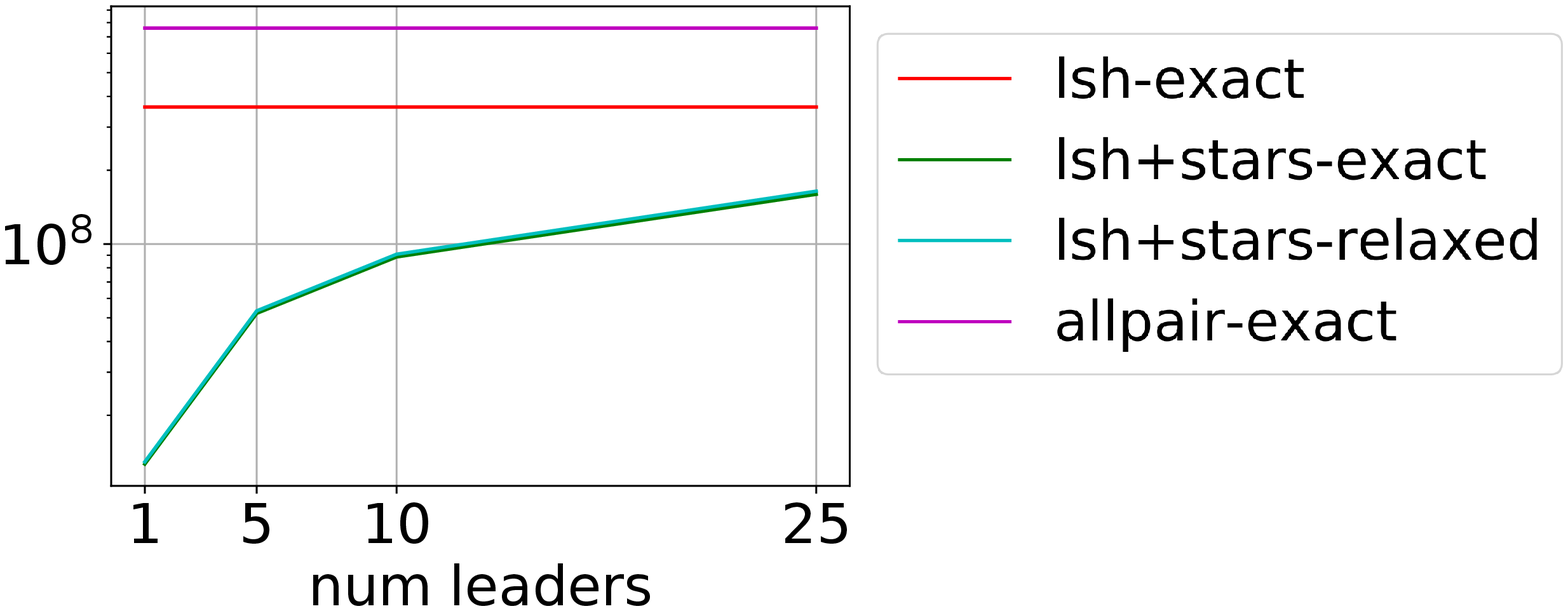}\\
Wikipedia&Amazon2m& \hspace{-2cm} MNIST\\
\end{tabular}
    \caption{\small The number of edges with similarity at least $0.5$ ($0.495$ for relaxed threshold) built by each algorithm with different number of leaders.}
    \label{fig:sparsity_different_leaders} 
\end{figure*}

\subsection{Running Time}
In this section, we investigate the running time of several experiments. 
We use \emph{total running time} to refer to the summation of running time of building edges over all machines.
We use \emph{real running time} to refer to the actual running time of the entire job, which includes I/O, shuffle and scheduling time etc.
Since real running time depends on the number of machines assigned to the job, I/O latency, and network condition etc., the total running time is usually less noisy.
For all experiments, the worker pool contains $1000$ number of machines (i.e., the maximum parallelism uses $1000$ machines, though a job might not get all of them due to scheduling).

\paragraph{Effect of similarity function.}
We report the relative total running time of Stars versions of algorithms and non-Stars versions of algorithms for Amazon2m for both mixture of similarities and the similarity learned by neural networks.
See Table~\ref{tab:lsh_amazon2m} for LSH-based algorithms, and see Table~\ref{tab:slsh_amazon2m} for SortingLSH-based algorithms.

\begin{table}
  \caption{Relative total running time of LSH-based algorithms for Amazon2m dataset for different similarity functions. 
  The actual total running time corresponding to relative total running time $1.00$ is $\sim 20$ hours.
  Note that the total running time is the summation of running time of computing edges over all machines.
  }
  \label{tab:lsh_amazon2m}
  \centering
  \begin{tabular}{lll}
    \toprule
         & Mixture of similarities     & Learned similarity \\
    \midrule
    LSH+non-Stars (\# sketches $R = 25$)    & $1.00$             & $3.34$     \\
    LSH+non-Stars (\# sketches $R = 400$)   & $11.05$            & $49.75$     \\
    LSH+Stars (\# sketches $R = 25$)        & $0.05$             & $0.51$  \\
    LSH+Stars (\# sketches $R = 400$)       & $0.87$             & $4.39$  \\
    \bottomrule
  \end{tabular}
\end{table}

\begin{table}
  \caption{Relative total running time of SortingLSH-based algorithms for Amazon2m dataset for different similarity functions. 
  The actual total running time corresponding to relative total running time $1.00$ is $\sim 38$ hours.
  Note that the total running time is the summation of running time of computing edges over all machines.
  }
  \label{tab:slsh_amazon2m}
  \centering
  \begin{tabular}{lll}
    \toprule
         & Mixture of similarities     & Learned similarity \\
    \midrule
    SortingLSH+non-Stars (\# sketches $R = 25$)    & $1.00$             & $22.39$     \\
    SortingLSH+non-Stars (\# sketches $R = 400$)   & $12.11$            & $126.16$     \\
    SortingLSH+Stars (\# sketches $R = 25$)        & $0.23$             & $2.65$  \\
    SortingLSH+Stars (\# sketches $R = 400$)       & $1.49$             & $16.38$  \\
    \bottomrule
  \end{tabular}
\end{table}

\paragraph{Experiments on 1B and 10B datasets.}
For Random1B and Random10B datasets, we only run algorithms with number of sketches $R=25$.
We report the relative total running time of Stars versions of algorithms and non-Stars versions of algorithms for both random1B and random10B datasets in Table~\ref{tab:relative_random}.

\begin{table}[h]
  \caption{Relative total running time of algorithms on Random1B and Random10B datasets. 
  The actual total running time corresponding to relative total running time $1.00$ is $\sim 2632$ hours.
  Note that the total running time is the summation of running time of computing edges over all machines.
  }
  \label{tab:relative_random}
  \centering
  \begin{tabular}{lll}
    \toprule
         & Random1B     & Random10B \\
    \midrule
    LSH+non-Stars (\# sketches $R = 25$)           & $1.000$             & $13.885$     \\
    SortingLSH+non-Stars (\# sketches $R = 400$)   & $0.076$             & $0.907$     \\
    LSH+Stars (\# sketches $R = 25$)               & $0.017$             & $0.178$  \\
    SortingLSH+Stars (\# sketches $R = 400$)       & $0.011$             & $0.118$  \\
    \bottomrule
  \end{tabular}
\end{table}

We also saw the speed-up of real running time on Random1B and Random10B datasets:
\begin{enumerate}
    \item For Random1B: All jobs of Stars-versions of algorithms finished in $1$ hour.
    SortingLSH+non-Stars finished in $1.5$ hours and LSH+non-Stars finished in $2$ hours.
    \item For Random10B: All jobs of Stars-versions of algorithms finished in $2$ hours.
    SortingLSH+non-Stars finished in $2$ hours and LSH+non-Stars finished in $23$ hours.
\end{enumerate}
For large datasets, the running time of finding edges dominate other overheads and thus we can observe a large speed-up as described above.

\end{document}